%% file: my_neurips_2025.tex
\PassOptionsToPackage{numbers,sort&compress}{natbib}
\documentclass{article}

\usepackage[preprint]{neurips_2025}

\usepackage[utf8]{inputenc} 
\usepackage[T1]{fontenc}    
\usepackage{hyperref}       
\usepackage{url}            
\usepackage{booktabs}       
\usepackage{amsfonts}       
\usepackage{nicefrac}       
\usepackage{microtype}      
\usepackage{xcolor}         

\usepackage{amsthm}

\newtheorem{RQ}{RQ}
\newtheorem{definition}{Definition}
\usepackage{algorithm}
\usepackage{algpseudocode}
\usepackage{booktabs}
\usepackage{array}
\newtheorem{assumption}{Assumption}
\newtheorem{theorem}{Theorem}
\newtheorem{lemma}{Lemma}
\newtheorem{corollary}{Corollary} 

\usepackage{amsmath}
\usepackage{graphicx}
\usepackage{amssymb}

\usepackage{enumitem}  

\usepackage{wrapfig}  

\usepackage{comment}
\usepackage{subcaption}  

\title{FedDAA: Dynamic Client Clustering for Concept Drift Adaptation in Federated Learning}

\newcommand{\sustechaffil}{
  Department of Computer Science and Engineering\\
  Southern University of Science and Technology\\
  Shenzhen, China 518055
}

\author{%
  Fu Peng \\
  \sustechaffil \\
  \texttt{pengf2022@mail.sustech.edu.cn} \\
  \And
  Ming Tang \\
  \sustechaffil \\
  \texttt{tangm3@mail.sustech.edu.cn} \\
}

\begin{document}

\maketitle
\input{sec/0_abstract}    
\input{sec/1_intro}
\input{sec/2_relatedwork}
\input{sec/3_problemformulation}
\input{sec/4_methodology}
\input{sec/5_convergence_analysis}
\input{sec/6_experiments}
\input{sec/7_conclusion}


\input{my_neurips_2025.bbl}
\include{sec/appendix}


\end{document}

%% file: sec/0_abstract.tex
\begin{abstract}
In federated learning (FL), the data distribution of each client may change over time, introducing both temporal and spatial data heterogeneity, known as concept drift. Data heterogeneity arises from three drift sources: real drift (shift in $P(y|\mathbf{x})$), virtual drift (shift in $P(\mathbf{x})$), and label drift (shift in $P(y)$). However, most existing FL methods addressing concept drift primarily focus on real drift. When clients experience virtual or label drift, these methods often fail to selectively retain useful historical knowledge, leading to catastrophic forgetting.
A key challenge lies in distinguishing different sources of drift, as they require distinct adaptation strategies: real drift calls for discarding outdated data, while virtual or label drift benefits from retaining historical data. Without explicitly identifying the drift sources, a general adaptation strategy is suboptimal and may harm generalization.
To address this challenge, we propose FedDAA, a dynamic clustered FL framework designed to adapt to multi-source concept drift while preserving valuable historical knowledge.
Specifically, FedDAA integrates three modules: a cluster number determination module to determine the optimal number of clusters; a real drift detection module to distinguish real drift and virtual/label drift; and a concept drift adaptation module to adapt to new data while retaining useful historical data.
We provide theoretical convergence guarantees, and experiments show that FedDAA achieves 7.84\%–8.52\% accuracy improvements over state-of-the-art methods on Fashion-MNIST, CIFAR-10, and CIFAR-100.

\end{abstract}

%% file: sec/1_intro.tex
\section{Introduction}
\label{sec:intro}

Federated learning (FL) is a decentralized machine learning approach that enables multiple clients to collaboratively train a shared model while keeping their data localized. Existing FL frameworks assume that client data distributions remain stable over time. However, this assumption often fails to hold in real-world scenarios. For example, user preferences for movie genres may evolve temporally; respiratory disease case counts commonly exhibit seasonal variations. As the data distributions of different clients could change at different time steps and even evolve in different directions (e.g., towards distinct distributions), this leads to \textit{data heterogeneity in both time and space} in FL. This phenomenon is referred to as concept drift in FL, shown in Fig.~\ref{fig:distributed concept drift}.

Data heterogeneity in either time or space can be triggered by three sources~\cite{DBLP:journals/tkde/LuLDGGZ19,DBLP:journals/ftml/KairouzMABBBBCC21,DBLP:journals/csur/GamaZBPB14}.
Specifically, suppose there are two datasets that follow joint probability distributions $P_1(\mathbf{x}, y)$ and $P_2(\mathbf{x}, y)$ respectively, where $\mathbf{x}$ is the feature vector and $y$ is the label. The two datasets may be collected from the same client at different time steps (temporal heterogeneity), or from different clients at the same time step (spatial heterogeneity).
Concept drift occurs if \( P_{1}(\mathbf{x},y) \neq P_{2}(\mathbf{x},y) \). The joint probability $ P(\mathbf{x},y)$ can be decomposed into two components, i.e., $P(\mathbf{x},y) = P(\mathbf{x})P(y|\mathbf{x})$ or $P(y)P(\mathbf{x}|y)$. Three sources \footnote{Readers may raise concerns about the absence of a fourth source: \( P_1(\mathbf{x}|y) \neq P_2(\mathbf{x}|y) \) while \( P_1(y) = P_2(y) \). The reasons are twofold: (i) existing works are primarily interested in predicting labels given features (i.e., $P(y|\mathbf{x})$)~\cite{DBLP:journals/tkde/LuLDGGZ19,DBLP:journals/ftml/KairouzMABBBBCC21,DBLP:conf/icml/GuoTL24}; (ii) for the fourth scenario, identical labels $y$ may correspond to different features $\mathbf{x}$ (e.g. due to seasonal variations or cultural heterogeneity)~\cite{DBLP:journals/ftml/KairouzMABBBBCC21}, which can be viewed as a form of virtual drift.} of data heterogeneity can be defined as:
(i) \textit{Real drift} arises when  $P_1(y|\mathbf{x}) \neq P_2(y|\mathbf{x})$  while $ P_1(\mathbf{x}) = P_2(\mathbf{x}) $; 
(ii) \textit{Virtual drift} occurs when \( P_1(\mathbf{x}) \neq P_2(\mathbf{x}) \) while \( P_1(y|\mathbf{x}) = P_2(y|\mathbf{x}) \);
(iii) \textit{Label drift} occurs when \( P_1(y) \neq P_2(y) \) while \( P_1(\mathbf{x}|y) = P_2(\mathbf{x}|y) \).
Among these three sources, real drift refers to changes in the conditional distribution \( P(y \mid \mathbf{x}) \), leading to a change in the decision boundary,\footnote{A decision boundary is the surface that separates different classes for a classification task in machine learning.} whereas virtual drift and label drift involve changes in the marginal distribution \( P(\mathbf{x}) \) or the label distribution \( P(y) \), without affecting the decision boundary. The change in the decision boundary significantly degrades model performance. Thus, we classify real drift as a separate category, while grouping virtual drift and label drift together.

\begin{wrapfigure}{r}{0.3\textwidth} 
    \centering
    \includegraphics[width=0.3\textwidth]{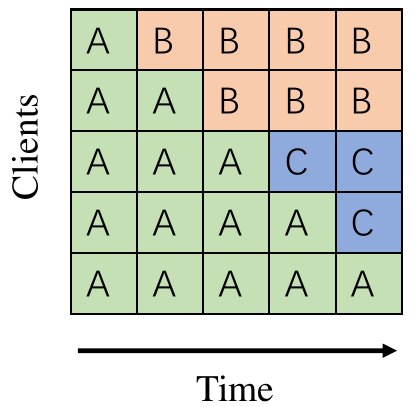} 
    \caption{FL under concept drift. There are two dimensional data heterogeneity: time and space. ``A'', ``B'' and ``C'' represent different data distributions.}
    \label{fig:distributed concept drift}
\end{wrapfigure}


%
To deal with data heterogeneity in space (i.e., across clients), existing works~\cite{DBLP:conf/icml/GuoTL24,DBLP:conf/nips/GhoshCYR20,DBLP:journals/tmc/WangXLXHZ23,DBLP:conf/ijcnn/BriggsFA20,DBLP:journals/tnn/SattlerMS21,DBLP:conf/aaai/RuanJ22,DBLP:conf/nips/MarfoqNBKV21,DBLP:journals/www/LongXSZWJ23,DBLP:journals/corr/abs-2201-07316,DBLP:journals/tnn/YanTW24} proposed clustered FL methods. Clustered FL is a framework that groups clients
into clusters based on their local data distributions to address
the data heterogeneity across clients. 
The idea of clustered FL stems from the fact that no single global model can perform well for all clients, especially when clients' data follow different conditional distributions $P(y|\mathbf{x})$ (i.e., decision boundaries)~\cite{DBLP:journals/tnn/SattlerMS21, DBLP:conf/aistats/JothimurugesanH23}.
Recently, \cite{DBLP:conf/icml/GuoTL24} proposes a principle of clustering: separating clients with different data conditional distributions into different clusters. 
However, existing clustered FL methods~\cite{DBLP:conf/icml/GuoTL24,DBLP:conf/nips/GhoshCYR20,DBLP:journals/tmc/WangXLXHZ23,DBLP:conf/ijcnn/BriggsFA20,DBLP:journals/tnn/SattlerMS21,DBLP:conf/aaai/RuanJ22,DBLP:conf/nips/MarfoqNBKV21,DBLP:journals/www/LongXSZWJ23,DBLP:journals/corr/abs-2201-07316,DBLP:journals/tnn/YanTW24} fail to consider temporal data heterogeneity and consequently do not adjust the number of clusters over time. 
When considering temporal heterogeneity, most existing FL concept drift methods~\cite{DBLP:conf/aistats/JothimurugesanH23,yang2024multi,DBLP:conf/nips/ChenX0LH24,zhou2024drift,DBLP:journals/mta/CasadoLCIRB22,DBLP:conf/bigdataconf/ChenCCR21,DBLP:conf/mobihoc/ZhangZXZ0CCY24,DBLP:conf/ijcnn/CanonacoBMR21,DBLP:conf/icml/PanchalCMMSMG23,saile2024client} primarily focus on real drift and often overlook the presence of other drift types such as virtual and label drift. As a result, these methods often mistakenly treat virtual and label drift as real drift, and discard historical data that could still be beneficial when virtual or label drift occurs. However, in such cases, the decision boundaries between historical and current data often remain similar, meaning that historical data can still provide useful knowledge and enhance model generalization. Consequently, these methods fail to simultaneously adapt to new data distributions and retain valuable historical knowledge, leading to catastrophic forgetting and making them suboptimal for handling complex and dynamic environments.

To address the above challenges, we aim to: (i) adaptively adjust the number of clusters, and (ii) distinguish real drift from virtual or label drift, enabling adaptation to new data distributions while retaining useful historical knowledge. Specifically, we formulate three research questions (RQs) as follows.




\begin{RQ}[Dynamic Number of Clusters]
         The number of clusters (the number of data conditional distributions $P(y|\mathbf{x})$) dynamically evolves over time in clustered FL. How to determine the optimal number of clusters at each time step?
\end{RQ}

\begin{RQ}[FL Concept Drift Detection and Adaptation] 
How can we detect real drift to distinguish it from virtual and label drift? Based on the distinction between real drift and virtual/label drift, how can we design two adaptation strategies: one for real drift and another for virtual/label drift?
\end{RQ}

\begin{RQ}[Convergence Analysis]
    What are the theoretical convergence guarantees in clustered FL under concept drift?
\end{RQ}

To resolve the aforementioned research questions, we propose FedDAA, a novel clustered FL approach that dynamically adapts to multi-source concept drift while preserving useful historical knowledge. Our approach consists of three main modules.
First, since the number of clusters evolves over time (\textbf{RQ~1}), we propose a module to dynamically determine the optimal number of clusters. 
To this end, we introduce the concept of a \textit{data prototype}, defined as a matrix that compactly represents the conditional distribution \(P(y|\mathbf{x})\) of a client's local dataset. 
Our module clusters these data prototypes to  determine both the optimal number of clusters and the corresponding prototype cluster centers at each time step.
Second, to distinguish real drift from virtual/label drift (\textbf{RQ~2}), we develop a real drift detection module. 
When a client’s historical and current prototypes are assigned to different clusters, determined by the shortest Euclidean distance to the prototype clustering center, real drift is detected. Otherwise, the client's conditional distribution $P(y|\mathbf{x})$ remains unchanged, indicating virtual/label drift or no drift.
Third, we design a concept drift adaptation module (\textbf{RQ~2}) that enables FL system to adapt to new distributions selectively preserving useful historical knowledge. This module ensures that clients experiencing real drift update models using only current data, while those without real drift incorporate both historical and current data to mitigate catastrophic forgetting.
Fourth, to address \textbf{RQ~3}, we conduct a convergence analysis of FedDAA, which provides theoretical guarantees on the stability and convergence of the training process under dynamic cluster restructuring and evolving client data distributions.

This paper presents the following contributions:
\begin{enumerate}
\item We investigate a realistic FL concept drift scenario involving  occurrence of multiple concept drift sources (real, virtual, and label drift) and formulate three key research questions.

\item We propose a novel clustered FL framework named FedDAA. To enable dynamic adjustment of the number of clusters over time, we introduce data prototypes, which capture clients' conditional data distributions. In addition, we design a real drift detection module and employ distinct adaptation strategies for real drift and virtual/label drift, allowing selective retention of useful historical knowledge.
We also provide a convergence analysis of FedDAA in the presence of a dynamically changing number of clusters.
\item Experimental results demonstrate that FedDAA achieves consistent improvements over state-of-the-art methods, with average accuracy gains of 8.52\%, 7.84\%, and 8.31\% on Fashion-MNIST, CIFAR-10, and CIFAR-100 respectively.
\end{enumerate}



%% file: sec/2_relatedwork.tex
\section{Related Works}
\label{sec:related works}

\textbf{FL under concept drift.}  
Early FL methods that address concept drift~\cite{zhou2024drift,DBLP:journals/mta/CasadoLCIRB22,DBLP:conf/bigdataconf/ChenCCR21,DBLP:conf/ijcnn/CanonacoBMR21,DBLP:conf/icml/PanchalCMMSMG23,saile2024client,DBLP:conf/mobihoc/ZhangZXZ0CCY24} typically assume that all clients experience real drift simultaneously and share a single data conditional distribution $P(y|\mathbf{x})$, thus requiring only one global model.
However, a more realistic scenario is that different clients may experience real drift at different times, resulting in the coexistence of multiple conditional  distributions $P(y|\mathbf{x})$ in FL at a given moment. In this case, the single-model
solution fails to handle data heterogeneity across clients.
To address this problem, recent FL concept drift methods~\cite{DBLP:conf/aistats/JothimurugesanH23,yang2024multi,DBLP:conf/nips/ChenX0LH24} try to create multiple models to address spatial data heterogeneity and simply retrain models to address temporal data heterogeneity.
While existing methods~\cite{DBLP:conf/aistats/JothimurugesanH23,yang2024multi,DBLP:conf/nips/ChenX0LH24} considering spatial and temporal data heterogeneity, two critical challenges remain unresolved.
First, they assume that FL exclusively experience real drift, while ignoring the potential occurrence of both virtual/label drift. Thus, their concept drift detection method can only detect changes in the data joint distribution $P(\mathbf{x},y)$, leading to misidentify virtual/label drift as real drift.
Second, when virtual/label drift occurs, they mistakenly discard historical data that could enhance model generalization, relying solely on the current data for concept drift adaptation. This leads to catastrophic forgetting.
In this work, we consider a more practical setting where virtual, real, and label drift may occur. Furthermore, we propose a novel framework to simultaneously adapt to new distributions and selectively retain useful historical knowledge.
\textbf{Clustered federated learning.} Clustered FL is a promising approach to address spatial data heterogeneity in FL. Foundational clustered FL methods~\cite{DBLP:conf/icml/GuoTL24,DBLP:conf/nips/GhoshCYR20,DBLP:journals/tmc/WangXLXHZ23,DBLP:conf/ijcnn/BriggsFA20,DBLP:journals/tnn/SattlerMS21,DBLP:conf/aaai/RuanJ22,DBLP:conf/nips/MarfoqNBKV21,DBLP:journals/www/LongXSZWJ23,DBLP:journals/corr/abs-2201-07316,DBLP:journals/tnn/YanTW24} aim to group clients with similar data distributions to improve model performance. 
The authors of \cite{DBLP:conf/nips/MarfoqNBKV21} propose a clustering-based approach by formulating and solving an optimization problem to handle federated multi-task learning. 
Building on this, the authors of \cite{DBLP:conf/aaai/RuanJ22} introduce FedSoft, a pioneering soft clustering method that assumes each client's data is drawn from multiple distributions.
However, recently the authors of~\cite{DBLP:conf/icml/GuoTL24} argue that only clients with different data conditional distributions are assigned to separate clusters to improve model generalization performance.  
They propose FedRC, a robust clustering framework that addresses spatial data heterogeneity in FL by dividing clients with different data conditional distributions into different groups, while explicitly considering three concept drift sources.
However, existing clustered FL methods~\cite{DBLP:conf/icml/GuoTL24,DBLP:conf/nips/GhoshCYR20,DBLP:journals/tmc/WangXLXHZ23,DBLP:conf/ijcnn/BriggsFA20,DBLP:journals/tnn/SattlerMS21,DBLP:conf/aaai/RuanJ22,DBLP:conf/nips/MarfoqNBKV21,DBLP:journals/www/LongXSZWJ23,DBLP:journals/corr/abs-2201-07316,DBLP:journals/tnn/YanTW24} do not consider temporal data heterogeneity.
In this work, we propose a novel framework based on clustered FL to address temporal data heterogeneity for FL under concept drift . 

%% file: sec/3_problemformulation.tex
\section{Problem Formulation}

\textbf{Concept drift definition and sources in FL.}

We consider an FL system with $K$ clients. The data distribution of clients changes over time.  The dataset at each client $k\in \mathcal{K}\triangleq\{1,2,...,K\}$ at each time $t \in \mathcal{T}\triangleq\{1,2,..., T\}$ is denoted by $S_{k}^{t}=\{Z_{k,0}^{t},Z_{k,1}^{t},...,Z_{k,N_{k}^{t}}^{t}\}$. $N_{k}^{t}$ is the number of data samples of client \(k\) at time step \(t\). The $i$-th data sample $Z_{k,i}^{t}$ is defined as $\{\mathbf{x}_{k,i}^{t},y_{k,i}^{t}\}$, where $\mathbf{x}_{k,i}^{t}$ represents the features and $y_{k,i}^{t}$ represents the labels. Suppose the dataset $S_{k}^{t}$ are sampled from a joint distribution $P_{k}^{t}(\mathbf{x},y)$.  Based on the concept drift survey~\cite{DBLP:journals/tkde/LuLDGGZ19}, we define temporal and spatial concept drift in FL. Concept drift can be mainly triggered by three sources in terms of time and space: real drift, virtual drift, label drift~\cite{DBLP:journals/tkde/LuLDGGZ19,DBLP:journals/ftml/KairouzMABBBBCC21}.
\begin{definition}[Concept Drift in FL]
    i) \textit{Spatial data heterogeneity}: For a given time \( t \) and two clients \( k_{1} \) and \( k_{2} \), concept drift occurs in FL if \( P_{k_{1}}^{t}(\mathbf{x},y) \neq P_{k_{2}}^{t}(\mathbf{x},y) \). 
    ii) \textit{Temporal data heterogeneity}: For a given time step \( t \) and a client \( k \), concept drift occurs in FL if \( P_{k}^{t-1}(\mathbf{x},y) \neq P_{k}^{t}(\mathbf{x},y) \).
\end{definition}

\textbf{Problem statement.}
In FL under concept drift, clients' data would follow different conditional distributions $P(y|\mathbf{x})$, resulting spatial data heterogeneity.  To address data heterogeneity across clients, we apply the clustered FL framework~\cite{DBLP:conf/icml/GuoTL24} to learn a set of global models by grouping clients with similar data conditional distributions $P(y|\mathbf{x})$. At each time step $t$, we learn a set of global models $\boldsymbol{w}_{c}^{t}$ for $c \in \mathcal{C}^t \triangleq\{1,2,..., C^t\}$, where $\mathcal{C}^t$ denotes the index set of clusters. 
Each global model $\boldsymbol{w}_c^{t}$ is corresponding to a cluster or a data conditional probability distribution $P(y|\mathbf{x})$.
Given a loss function $f(\boldsymbol{w}_c^{t};Z_{k}^{t})$, the population loss $F_k^{t}(\boldsymbol{w}_c^{t})$ of a client $k$ for the global model $\boldsymbol{w}_c^t$ at time $t$ is the expected loss for data following distribution $P_{k}^{t}(Z_k^t)$:  $F_k^{t}(\boldsymbol{w}_c^{t}) = \mathbb{E}_{Z_k^t\sim P_{k}^{t}}[f(\boldsymbol{w}_c^{t};Z_{k}^{t})]$.
Following previous works~\cite{DBLP:conf/icml/GuoTL24,DBLP:conf/nips/MarfoqNBKV21}, we suppose among the local dataset $S_k^t$, there are $N_{k,c}^{t}$ data points belong to cluster $c$. Here, we define $\alpha_{k,c}^{t} = N_{k,c}^t / N_k^t \in [0,1]$, which is the proportion of data from client $k$ that belongs to cluster $c$. We have $\sum_c \alpha_{k,c}^{t} = 1$. 
For each time step \(t\), our goal is to train \(C^{t}\) new global models that can both adapt to the current data and incorporate useful historical knowledge from clients. Specifically, at each time \(t\), we aim for the \(C_t\) global models to perform well on the current data at time step \(t\) as well as on the historical data from time \(t-1\). Accordingly, we aim to minimize the following objective over $T$ time steps and $K$ clients:

    \begin{equation}\label{eq: objective-2}
    \begin{aligned}
        &\sum\limits_{t=1}^{T}\sum\limits_{k=1}^{K}\sum\limits_{c\in \mathcal{C}_t} \alpha_{k,c}^{t-1} F_k^{t-1}(\boldsymbol{w}_c^{t})+\alpha_{k,c}^{t} F_k^{t}(\boldsymbol{w}_c^{t})\\
        \end{aligned}
    \end{equation}




%% file: sec/4_methodology.tex
\begin{figure}[t]
    \centering
    \includegraphics[width=1\textwidth]{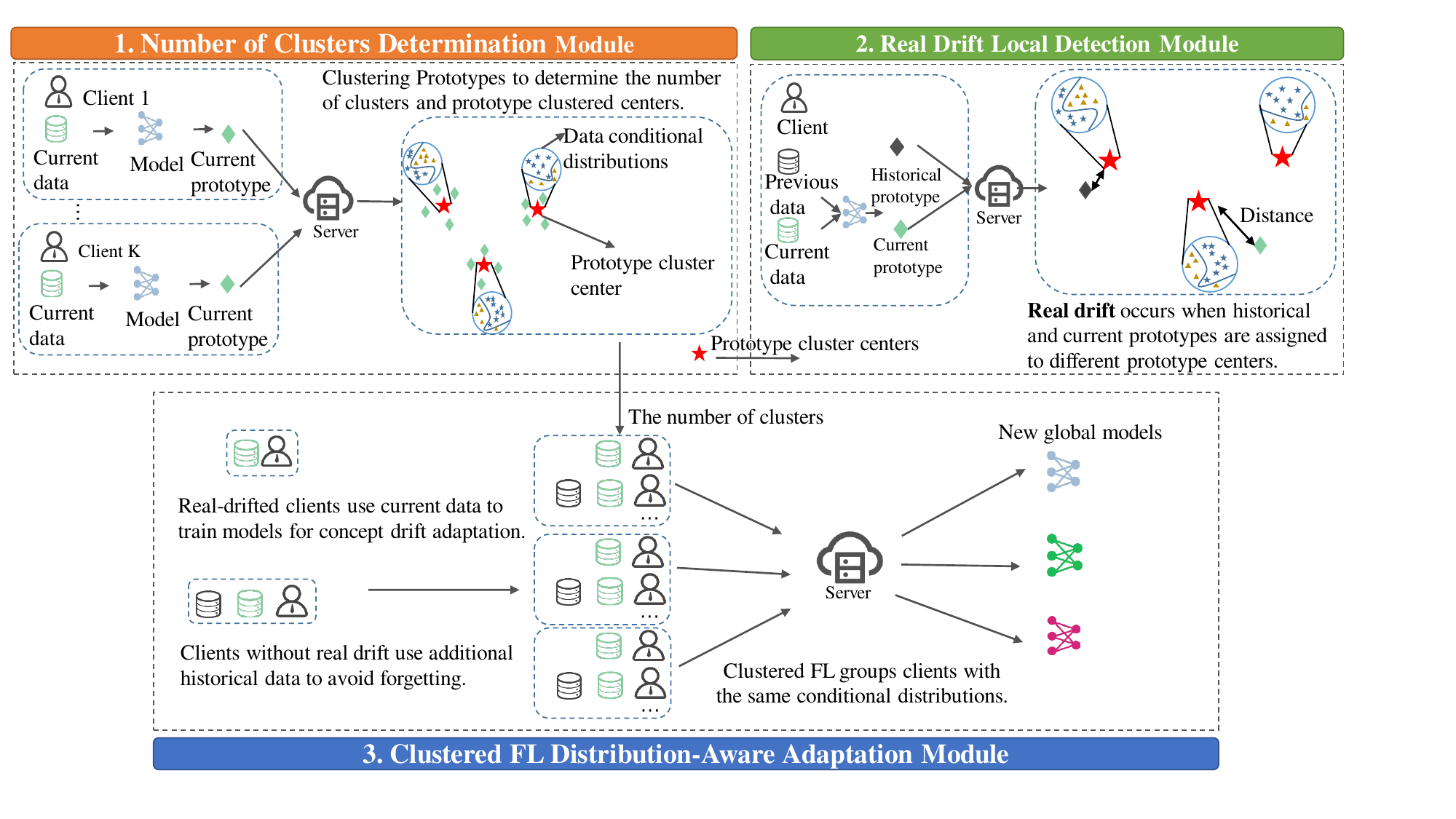}
    \caption{An overview of the proposed framwork FedDAA.}
    \label{fig:algorithm framework}
\end{figure}
\section{Methodology}\label{sec:Methodology}
In this section, we first provide an overview of FedDAA. We then describe each of its three core components individually, and conclude with the design of the full algorithm.

\subsection{Overview}
In FL, real drift may occur at different times across clients, resulting in both temporal and spatial data heterogeneity.
To address these challenges, we propose FedDAA, a novel clustered FL framework that considers multiple sources of concept drift. FedDAA is designed to simultaneously adapt to new data distributions while selectively retaining useful historical knowledge.
To mitigate spatial heterogeneity, we group clients with similar conditional distributions $P(y|\mathbf{x})$ to collaboratively train multiple global models.
To handle temporal heterogeneity, we detect and distinguish real drift from virtual and label drift, followed by tailored adaptation strategies.
FedDAA contains the following three main modules, shown in Fig.~\ref{fig:algorithm framework}.
i) Number of Clusters Determination (NCD) module. The NCD module determines the number of clusters for each time step by clustering client prototypes that represent their conditional distributions.
ii) Real Drift Local Detection (RDLD) module. The RDLD module detects whether a client experiences real drift by obtaining its data prototypes at current and previous time steps and identifying changes in their prototype cluster assignments.
iii) Clustered FL Distribution-Aware Adaptation (DAA) module. The DAA module implements adaptive strategies tailored to the sources of drift detected in each client, enabling effective model adaptation while preserving valuable historical knowledge.

\subsection{Number of Clusters Determination}\label{sec:Determine the Cluster Number}
Due to the number of clusters will probably changes over time, we propose the Number of Clusters Determination (NCD) module to determine the optimal number of clusters for each time step $t$. First, we design a data prototype to capture the data conditional distribution information for each client. 
Second, we cluster data prototypes from all clients to determine the optimal number of clusters. In the following, we give the definition of data prototype of a client.

A data prototype \( \mathbf{P}_k^{t} \) of a client $k$ at time $t$ is defined as a matrix that encapsulates the conditional distribution information of dataset $S_{k}^{t}$, serving as a compact representation of the data conditional distribution \(P(y|\mathbf{x})\). The data prototype of a client $k$ for a dataset $S_{k}^{t}$ is defined as follows.
\begin{definition}[Data Prototype of Client \(k\) for Dataset \(S_{k}^{t}\)  at Time Step \(t\)]
    We define the data prototype \(\mathbf{P}_{k}^{t} \in \mathbb{R}^{R \times R}\) for dataset \(S_{k}^{t}\) as
    \begin{multline}
    \mathbf{P}_{k}^{t} = \Biggl[ \frac{1}{N_{k,1}^{t}} \sum_{i=1}^{N_{k}^{t}} \boldsymbol{1}_{\{y_{k,i}^{t}=1\}} \, h(\boldsymbol{w}^{t}, \mathbf{x}_{k,i}^{t}), \; \cdots, \; \frac{1}{N_{k,R}^{t}} \sum_{i=1}^{N_{k}^{t}} \boldsymbol{1}_{\{y_{k,i}^{t}=R\}} \, h(\boldsymbol{w}^{t}, \mathbf{x}_{k,i}^{t}) \Biggr],
    \label{eq:prototype}
    \end{multline}
    where \(R\) denotes the total number of classes in the dataset \(S_{k}^{t}\),  and \(N_{k,r}^{t}\) represents the number of samples in \(S_{k}^{t}\) that belong to the \(r\)-th class for each \(r=1,\ldots,R\). The indicator function \(\boldsymbol{1}_{\{y_{k,i}^{t}=r\}}\) takes the value $1$ when \(y_{k,i}^{t}=r\) and $0$ otherwise. \(h(\boldsymbol{w}^{t}, \mathbf{x}_{k,i}^{t})\) is the output vector produced by the model parameterized by \(\boldsymbol{w}^{t}\) for the input \(\mathbf{x}_{k,i}^{t}\).

\end{definition}


After obtaining data prototypes, each client sends its prototype to the server. The server clusters these prototypes (e.g., via $K$-means algorithm) and determines the optimal cluster number using silhouette scores. We then use the cluster centers in the RDLD module to detect real drift per client.
Pseudocode for the NCD module is in Appendix~\ref{sec:appendix Determine the Cluster Number}.

\subsection{Concept Drift Detection and Adaptation}\label{sec:Real Drift Detection}
\textbf{Real drift detection.} To detect whether a client experiences real drift, we propose a Real Drift Local Detection (RDLD) module. The main idea of RDLD module is to utilize the the prototype cluster centers from NCD module in Section~\ref{sec:Determine the Cluster Number} to perform detection, where the cluster centers contain knowledge of different data conditional distributions.  First, for each client $k$, we use Eq.~\eqref{eq:prototype} to calculate the data prototypes $\mathbf{P}^{t}_{k}$ based on dataset $S_{k}^{t}$ at time $t$ and $\mathbf{P}^{t-1}_{k}$ based on dataset $S_{k}^{t-1}$ at time $t-1$, respectively.
Second, we compute the distances between $\mathbf{P}^{t}_{k}$ and $\mathbf{P}^{t-1}_{k}$ to the prototype cluster centers obtained by the NCD module, respectively.  Let $\mathbf{P}_{c}^{t}$ to denote the prototype cluster centers, where $c=1,...,C^{t}$, \(C^{t}\) is the number of clusters at time $t$. we use the following equation to calculate distances between $\mathbf{P}^{t}_{k}$ to the prototype cluster center $\mathbf{P}_{c}^{t}$.
\begin{align}\label{eq:Euclidean distance}
d(\mathbf{P}^{t}_{k}, \mathbf{P}_{c}^{t}) 
&= \|\mathbf{P}^{t}_{k} - \mathbf{P}_{c}^{t}\|_2.
\end{align}
Third, we assign each prototype to a cluster by selecting the prototype cluster center that is closest in Euclidean distance. Let \(c_{k}^{t-1}\) be the cluster to which prototype \(\mathbf{P}_{k}^{t-1}\) is assigned, and \(c_{k}^{t}\) be the cluster for prototype \(\mathbf{P}_{k}^{t}\). Specifically, these clusters are determined by the following equations:
\begin{equation}
\begin{aligned}
c_{k}^{t-1} &= \mathop{\arg\min}_{c \in \{1,\dots,C^{t}\}} d(\mathbf{P}_{k}^{t-1}, \mathbf{P}_{c}^{t}), \quad
c_{k}^{t}   &= \mathop{\arg\min}_{c \in \{1,\dots,C^{t}\}} d(\mathbf{P}_{k}^{t}, \mathbf{P}_{c}^{t}).
\end{aligned}
\label{eq:cluster_assignment}
\end{equation}
Finally, if \(c_{k}^{t-1} \neq c_{k}^{t}\) for a client \(k\), then the change in cluster assignment indicates that client \(k\) is experiencing real drift. The pseudocode for the RDLD module is provided in Appendix~\ref{sec:appendix Real Drift Detection}.

\textbf{Concept drift adaptation.} To help clients adapt to real and virtual/label drift, we propose Distribution-Aware Adaptation (DAA) module. When real drift is detected in a client $k$, client $k$ will update their local models using the current dataset from time step $t$; otherwise, client $k$ conducts local updating using both datasets from time steps $t$ and $t-1$. 

\subsection{FedDAA}
In this section, we present FedDAA that contains the three main modules described above. 
The pseudocode for FedDAA is provided in Appendix~\ref{sec:appendix Concept Drift Adaptation}.
For each time step $t$, we separate clients with different conditional distributions $P(y^{t}|\mathbf{x}^{t})$ into distinct clusters, while keeping clients with the same conditional distribution in one cluster in clustered FL~\cite{DBLP:conf/icml/GuoTL24}.
First, we use the NCD module to determine the number of clusters $C^{t}$. Then, we use the RDLD module to detects whether each client experiences real drift. Third, based on the drift type identified, the DAA module decides whether to incorporate historical data for local model training.
After obtaining the number of clusters and the training data for each client, we want to learn multiple global models  $\boldsymbol{w}_1^{t}, \dots, \boldsymbol{w}_{C}^{t}$  and the clustering weights $ \hat{\alpha}_{k,c}^{t}$ to cluster $c$ (approximation of $ \alpha_{k,c}^{t}$ in Eq.~\eqref{eq: objective-2}, which is unknown in practical).
Let vector $\mathbf{w}^t=[\boldsymbol{w}_1^{t}, \dots, \boldsymbol{w}_{C}^{t}]$ and $\boldsymbol{\hat{\alpha}}^{t} = [\hat{\alpha}_{1,1}^{t}, \dots, \hat{\alpha}_{K,C}^{t}] \in~\mathbb{R}^{K \times C}$.
Typically, $\sum_{c=1}^{C} \hat{\alpha}_{k,c}^{t} = 1$,  $\forall k $ clients. 
To separate clients with different conditional distributions into distinct clusters, the objective function
 $ \mathcal{L}(\mathbf{w}^{t},\boldsymbol{\hat{\alpha}}^{t})$ should satisfy the following two properties.



\textbf{Avoiding real drift within cluster.} Maximizing $ \mathcal{L}(\mathbf{w}^{t},\boldsymbol{\hat{\alpha}}^{t})$ should prevent real drift within the same cluster.
If $(\mathbf{x}^{t}, y^{t})$ exhibits a real drift with respect to cluster  $c$, the conditional probability $ \mathcal{P}(y^{t}|\mathbf{x}^{t}; \boldsymbol{w}_{c}^{t}) $ will be small.

\textbf{Decoupling real drift with virtual/label drift.} 
Maximizing $ \mathcal{L}(\mathbf{w}^{t}, \boldsymbol{\hat{\alpha}}^{t}) $ should decouple real drift from virtual or label drift. 
If \( (\mathbf{x}^{t}, y^{t}) \) has virtual or label drift with respect to cluster  $c$, the marginal probability \( \mathcal{P}(y^{t}; \boldsymbol{w}_{c}^{t}) \) or \( \mathcal{P}(\mathbf{x}^{t}; \boldsymbol{w}_{c}^{t}) \) will be small. 

 Then the objective function at time step $t$ is presented in the following Eq.~\eqref{eq:objective funtion}. The key terms $\mathcal{I}(\mathbf{x}^{t}_{k,i}, y^{t}_{k,i}; \boldsymbol{w}_{c}^{t})$  are designed to satisfy the two properties outlined above.  Specifically, we define $\mathcal{I}(\mathbf{x}^{t}_{k,i}, y^{t}_{k,i}; \boldsymbol{w}_{c}^{t}) = \frac{\mathcal{P}(\mathbf{x}^{t}_{k,i}, y^{t}_{k,i}; \boldsymbol{w}_{c}^{t})}{\mathcal{P}(\mathbf{x}^{t}_{k,i}; \boldsymbol{w}_{c}^{t}) \mathcal{P}(y^{t}_{k,i}; \boldsymbol{w}_{c}^{t})} = \frac{\mathcal{P}(y^{t}_{k,i}|\mathbf{x}^{t}_{k,i}; \boldsymbol{w}_{c}^{t})}{\mathcal{P}(y^{t}_{k,i}; \boldsymbol{w}_{c}^{t})} = \frac{\mathcal{P}(\mathbf{x}^{t}_{k,i}|y^{t}_{k,i}; \boldsymbol{w}_{c}^{t})}{\mathcal{P}(\mathbf{x}^{t}_{k,i}; \boldsymbol{w}_{c}^{t})}$.

\begin{equation}
\begin{aligned}
\label{eq:objective funtion}
        \mathcal{L}(\mathbf{w}^{t}, \boldsymbol{\hat{\alpha}}^{t}) &= \frac{1}{N^{t}} \sum_{k=1}^K \sum_{i=1}^{N_{k}^{t}} \ln \left( \sum_{c=1}^{C^{t}} \hat{\alpha}_{k,c}^{t} \mathcal{I}(\mathbf{x}^{t}_{k,i}, y^{t}_{k,i}; \boldsymbol{w}_{c}^{t}) \right),\\
        \text{s.t.} \sum_{c=1}^{C^{t}} \hat{\alpha}_{k,c}^{t} &= 1, \, \forall k,
\end{aligned}
\end{equation}
where $N^{t} = \sum_{k=1}^{K}N_{k}^{t}$ for each time $t$. The detailed optimization steps for Eq.\eqref{eq:objective funtion} are provided in Appendix~\ref{appendix: Proofs of objective funtion}. 

%% file: sec/5_convergence_analysis.tex
\section{Convergence Analysis}\label{sec: Convergence Analysis}
In this section, we present the convergence analysis of our approach FedDAA. For each time step $t$, we run FedDAA with $\tau$ communication rounds to retrain models to adapt to concept drift. We focus on the training at time step $t$, and therefore omit all superscripts $t$ in the following.
We make the following assumptions on the local loss function $f(\mathbf{x},y,\boldsymbol{w})$. Assumptions~\ref{assumption: our 1}-\ref{assumption: our 6} are frequently used for non-convex analysis in the literature~\cite{DBLP:conf/nips/MarfoqNBKV21,DBLP:conf/icml/GuoTL24,DBLP:conf/iclr/ChenLSH19,DBLP:conf/nips/MertikopoulosHK20,DBLP:conf/iclr/LiHYWZ20,DBLP:conf/nips/WangLLJP20}. Assumptions~\ref{assumption: our 1}-~\ref{assumption: our 5} are used in the standard analysis and Assumption~\ref{assumption: our 6} formalizes the degree of dissimilarity among local objective functions. 

\begin{assumption}[Smoothness Assumption]\label{assumption: our 1} 
    Assume functions $ f(\boldsymbol{w}) $ are L-smooth, i.e. $ \nabla f(\boldsymbol{w}_1) - 
\nabla f(\boldsymbol{w}_2) \| \leq L \| \boldsymbol{w}_1 - \boldsymbol{w}_2 \| $.
\end{assumption}
\begin{assumption}[Bounded Gradient Assumption]\label{assumption: our 2} 
    Assume that the gradient of local objective functions $ \nabla f(\boldsymbol{w}) $ are bounded, i.e. 
$ \mathbb{E} \left[ \| \nabla f(\boldsymbol{w}) \|^2 \right] = \frac{1}{N} \sum_{k=1}^K \sum_{i=1}^{N_k} \| \nabla f(\mathbf{x}_{k,i}, y_{k,i}, \boldsymbol{w}) \|^2 \leq \sigma^2 $. 
\end{assumption}

\begin{assumption}[Unbiased Gradients and Bounded Variance] Each client $k$, $k\in \mathcal{K}=\{1,2,...,K\}$, can sample a random batch $\xi$ from $S_k$ and compute an unbiased estimator $g_k(\xi, \boldsymbol{w}_c)$ of the local gradient with bounded variance, i.e., 
\begin{equation}
    \mathbb{E}_{\xi} [g_k(\xi, \boldsymbol{w}_c)] = \frac{1}{N_k} \sum_{k=1}^{N_k} 
\nabla_{\boldsymbol{w}_c} f(\mathbf{x}_{k,i}, y_{k,i}, \boldsymbol{w}_c),
\end{equation}
\begin{equation}
    \mathbb{E}_{\xi} \left\| g_k(\xi, \boldsymbol{w}_c) - \frac{1}{N_k} \sum_{i=1}^{N_k} 
\nabla_{\boldsymbol{w}_c} f(\mathbf{x}_{k,i}, y_{k,i}, \boldsymbol{w}_c) \right\|^2 \leq \delta^2.
\end{equation}
\label{assumption: our 5}
\end{assumption}

\begin{assumption}[Bounded Dissimilarity]\label{assumption: our 6} There exist $\beta$ and $G$ such that:
\begin{equation}
    \begin{aligned}
    &\sum_{k=1}^{K} \frac{1}{N} \left\| \sum_{i=1}^{N_k} \sum_{c=1}^{C} \nabla_{\boldsymbol{w}_c} f(\mathbf{x}_{k,i}, y_{k,i}, \boldsymbol{w}_c) \right\|^2 \leq G^2 + \beta^2 \left\| \frac{1}{N} \sum_{k=1}^{K} \sum_{i=1}^{N_k} \sum_{c=1}^{C} \nabla_{\boldsymbol{w}_c} f(\mathbf{x}_{k,i}, y_{k,i}, \boldsymbol{w}_c) \right\|^2,
    \end{aligned}
\end{equation}

where $N=\sum_{k=1}^{K}N_{k}$.
\end{assumption}

\begin{assumption}[Expectation of the Number of Clusters]\label{assumption: our 3} 
    There is an optimal number of clusters $C^{*}$. Assume that the number of clusters $C$ follows a certain distribution, and $\mathbb{E}C=C^{*}$. 
\end{assumption}


\begin{theorem}
\label{corollary: federated our}
     Under Assumptions~\ref{assumption: our 1}, \ref{assumption: our 2}, \ref{assumption: our 5}, \ref{assumption: our 6} and \ref{assumption: our 3}, clients use SGD as local solver with learning rate $\eta \leq \frac{1}{\sqrt{\tau}}$, and run our method for $\tau$ communication rounds and $E$ local epochs for each local update. Let $ \Delta_{\xi,c}^{2} = \mathbb{E}_{\xi}\|\boldsymbol{w}_{c}^{\tau}-\boldsymbol{w}_{c}^{0}\|^{2}$ for all $c$, and $\Delta_{\xi}^2 = \max \{\Delta_{\xi,c}^{2},c=1,2,3,...,C.\}$. We have
\begin{equation}
    \mathbb{E}_{C}\mathbb{E}_{\xi}[\mathcal{L}(\mathbf{w}^{\tau}, \hat{\boldsymbol{\alpha}}^{\tau},C^{*}) - \mathcal{L}(\mathbf{w}^{\tau}, \hat{\boldsymbol{\alpha}}^{\tau},C)] \leq \mathcal{L}(\mathbf{w}^{*}, \hat{\boldsymbol{\alpha}}^{*},C^{*}) -\mathcal{L}(\mathbf{w}^{0}, \hat{\boldsymbol{\alpha}}^{0},C^{*})+B,
\end{equation}
where 
\begin{equation}
    B= -\frac{EC^{*}\Delta_{\xi}^{2}}{4\eta}+ 2(E\eta^2L(\eta L+1)\sigma^2+2E^{3}\eta^{3}L^{2}G^{2}).
\end{equation}
\end{theorem}
The proof of Theorem~\ref{corollary: federated our} is in appendix~\ref{sec:appendix Convergence Analysis for Federated Clustering Method}.   Theorem~\ref{corollary: federated our} shows that using the number of clusters $C$ determined by FedDAA, $\mathbb{E}_{C}\mathbb{E}_{\xi}[\mathcal{L}(\mathbf{w}^{\tau}, \hat{\boldsymbol{\alpha}}^{\tau},C^{*}) - \mathcal{L}(\mathbf{w}^{\tau}, \hat{\boldsymbol{\alpha}}^{\tau},C)]$ can be bounded by a constant, indicating the convergence of FedDAA.

%% file: sec/6_experiments.tex
\section{Experiments}

\subsection{Experiments Setup}\label{sec: Experiments Setup}

\textbf{Datasets and models.} We conduct experiments on three datasets: Fashion-MNIST~\cite{DBLP:journals/corr/abs-1708-07747}, CIFAR-10~\cite{krizhevsky2009learning}, and CIFAR-100. For Fashion-MNIST, we employ a two-layer CNN model. We utilize a ResNet-18~\cite{DBLP:conf/cvpr/HeZRS16} model for CIFAR-10, and a MobileNet~\cite{DBLP:journals/corr/HowardZCKWWAA17} for CIFAR-100.

\textbf{Baselines.} 
We compare our method against methods in the following categories:  
i) single-model methods: FedAvg~\cite{DBLP:conf/aistats/McMahanMRHA17} with simple retraining, Adaptive-FedAvg~\cite{DBLP:conf/ijcnn/CanonacoBMR21}, and Flash~\cite{DBLP:conf/icml/PanchalCMMSMG23};  
ii) multiple-model methods: FedDrift~\cite{DBLP:conf/aistats/JothimurugesanH23}, FedRC~\cite{DBLP:conf/icml/GuoTL24} with simple retraining, and Oracle.

\textbf{Sources of concept drift.} To construct three sources of concept drift, we use the following methods to generated datasets. i)~Real drift. We change the labels of clients' data to change data conditional distribution $P(y|\mathbf{x})$, e.g, change label $y$ to $(y+1)\bmod R$ as in previous works~\cite{DBLP:conf/aistats/JothimurugesanH23,DBLP:conf/icml/GuoTL24,DBLP:conf/icml/PanchalCMMSMG23}. ii)~Virtual drift. We rotate the images in clients' data to change data marginal distribution $P(\mathbf{x})$, as in previous works~\cite{DBLP:conf/iclr/NguyenTGTB22,DBLP:conf/nips/NguyenTL22,DBLP:conf/iclr/WangM23,DBLP:conf/uai/PhamZ024}. iii)~Label drift. We employ the Dirichlet distribution to generate datasets with different marginal distribution $P(y)$, as in previous works~\cite{DBLP:journals/corr/abs-1909-06335,DBLP:conf/icml/GuoTL24,DBLP:conf/nips/MarfoqNBKV21,DBLP:conf/icml/PanchalCMMSMG23}.

\textbf{Concept drift settings.} We run 6 time steps in total. In each time steps, clients obtain new data and retrain models. i) To simulate the scenario where the number of clusters varies over time, we increase the number of data conditional distributions of all clients from 2 to 4 between time steps 1 and 3, and maintain 4 conditional distributions from time step 4 to time step 6. ii) To simulate a scenario involving both real drift and label drift, clients collect new datasets with new conditional distributions $P(y|\mathbf{x})$ and marginal label distributions $P(y)$ at each time step. In addition, to simulate virtual drift, we rotate the input images at each time step: at time step 1, images are not rotated; at time step 2, they are rotated by 120 degrees; and at time step 3, by 240 degrees. This rotation pattern repeats every three time steps.  All results are average over 5 runs (mean ± std). Appendix~\ref{sec:appendix Experiment} has more details.

\begin{table}[t]
    \caption{Comparison with baselines. Average test accuracy across all clients and all time steps.}
    \centering
    \begin{tabular}{@{}cccc@{}}
        \toprule
        Method                     & F-MNIST           & CIFAR-10           & CIFAR-100 \\ \midrule
        FedAvg                     & $42.53 \pm 0.34$  & $35.21 \pm 0.51$   & $27.15 \pm 0.62$ \\
        A-FedAvg                   & $51.15 \pm 0.43$  & $42.44 \pm 0.58$   & $34.11 \pm 0.48$ \\
        Flash                      & $50.22 \pm 0.40$  & $43.58 \pm 0.72$   & $33.83 \pm 0.55$ \\ \cline{1-4}
        FedRC                      & $56.82 \pm 0.61$  & $52.88 \pm 0.64$   & $37.52 \pm 0.69$ \\ 
        FedDrift                   & $62.52 \pm 0.57$  & $58.18 \pm 0.80$   & $44.36 \pm 0.75$ \\
        Ours                       & $\mathbf{71.04 \pm 0.42}$ & $\mathbf{66.02 \pm 0.59}$ & $\mathbf{52.67 \pm 0.50}$ \\ \cline{1-4}
        Oracle                     & $72.26 \pm 0.29$  & $66.62 \pm 0.47$   & $52.96 \pm 0.38$ \\ \bottomrule
    \end{tabular}
    \label{table:Comparison with baselines.}
\end{table}

\begin{table}[t]
    \caption{Performance comparison on forgetting mitigation. Test accuracy across different tasks at time step $t=6$ on CIFAR-10. Here, (1, $0^\circ$) is short for (concept 1, 0 degrees). "Concepts 1--4" refer to different conditional distributions. }
    \centering
    \resizebox{\linewidth}{!}{
    \begin{tabular}{@{}cccccccc@{}}
        \toprule
        Tasks                     & FedAvg & A-FedAvg & Flash & FedRC & FedDrift & Ours & Oracle \\ \midrule
        (1, $0^\circ$)       & 9.97 ± 0.42 & 11.83 ± 0.77 & 10.26 ± 0.58 & 9.97 ± 0.36 & 36.72 ± 0.41 & \textbf{65.61 ± 0.29} & 69.20 ± 0.66 \\
        (1, $120^\circ$)     & 9.94 ± 0.73 & 11.94 ± 0.61 & 11.94 ± 0.14 & 9.94 ± 0.39 & 40.64 ± 0.44 & \textbf{69.02 ± 0.12} & 68.87 ± 0.52 \\
        (1, $240^\circ$)     & 7.67 ± 0.81 & 10.10 ± 0.24 & 10.10 ± 0.65 & 7.67 ± 0.33 & 69.46 ± 0.70 & \textbf{70.71 ± 0.46} & 70.15 ± 0.20 \\ \cline{1-8}
        (2, $0^\circ$)       & 20.25 ± 0.48 & 14.14 ± 0.89 & 16.24 ± 0.57 & 20.25 ± 0.62 & 36.95 ± 0.30 & \textbf{65.50 ± 0.59} & 64.33 ± 0.76 \\
        (2, $120^\circ$)     & 18.80 ± 0.18 & 14.94 ± 0.63 & 13.56 ± 0.44 & 18.82 ± 0.53 & 44.36 ± 0.91 & \textbf{68.62 ± 0.27} & 70.32 ± 0.37 \\
        (2, $240^\circ$)     & 43.24 ± 0.39 & 24.75 ± 0.58 & 24.65 ± 0.36 & 43.24 ± 0.41 & 69.62 ± 0.69 & \textbf{70.75 ± 0.12} & 70.31 ± 0.43 \\ \cline{1-8}
        (3, $0^\circ$)       & 19.81 ± 0.22 & 18.98 ± 0.73 & 19.20 ± 0.63 & 19.81 ± 0.21 & 27.87 ± 0.39 & \textbf{60.23 ± 0.25} & 58.96 ± 0.70 \\
        (3, $120^\circ$)     & 20.40 ± 0.60 & 20.80 ± 0.48 & 21.60 ± 0.66 & 20.40 ± 0.34 & 48.02 ± 0.81 & \textbf{70.55 ± 0.49} & 70.81 ± 0.57 \\
        (3, $240^\circ$)     & 37.03 ± 0.40 & 36.08 ± 0.58 & 37.09 ± 0.22 & 37.03 ± 0.27 & 70.93 ± 0.90 & \textbf{71.52 ± 0.16} & 71.45 ± 0.51 \\ \cline{1-8}
        (4, $0^\circ$)       & 11.07 ± 0.37 & 15.21 ± 0.85 & 16.35 ± 0.30 & 11.07 ± 0.46 & 35.65 ± 0.62 & \textbf{64.41 ± 0.43} & 65.43 ± 0.75 \\
        (4, $120^\circ$)     & 10.67 ± 0.49 & 16.03 ± 0.91 & 17.56 ± 0.63 & 10.67 ± 0.18 & 41.13 ± 0.47 & \textbf{69.06 ± 0.24} & 67.95 ± 0.59 \\
        (4, $240^\circ$)     & 10.56 ± 0.33 & 25.71 ± 0.42 & 26.12 ± 0.51 & 10.56 ± 0.38 & 70.46 ± 0.54 & \textbf{71.21 ± 0.29} & 71.31 ± 0.61 \\ \bottomrule
    \end{tabular}
    }
    \label{table:Comparison with baselines in avoiding forgetting.}
\end{table}





\subsection{Experimental Results}
\textbf{Comparison with baselines using average accuracy over time.} 
The results in Table~\ref{table:Comparison with baselines.} show that our method consistently outperforms all baseline methods across three  datasets. 
Compared to the strongest baseline (FedDrift), our method improves accuracy by 8.52\% on Fashion-MNIST, 7.83\% on CIFAR-10, and 8.31\% on CIFAR-100, demonstrating the effectiveness of our dynamic clustering and distribution-aware adaptation for multi-source concept drift.
Compared to single-model baselines like FedAVG and A-FedAvg, our method achieves higher accuracy by grouping clients with different conditional distributions into separate clusters, each with its own specialized global model.  
Moreover, even compared to prior clustered FL approaches (e.g., FedRC and FedDrift), our method exhibits better adaptability to multi-source drift by explicitly detecting real drift and applying distinct adaptation strategies for real and virtual/label drift.
The small performance gap between our method and the Oracle, which assumes full knowledge of client groupings and data dynamics, further highlights the robustness of our approach.
\begin{table}[t]    
    \caption{Ablation study on different modules of our method. Average accuracy for our method across all clients and all time steps.``w/o'' stands for ``without''.}
    \centering
    \begin{tabular}{@{}p{2cm}ccc@{}}
        \toprule
        Method            & F-MNIST       & CIFAR-10      & CIFAR-100    \\ \midrule
        w/o NCD          & 58.26 $\pm$ 0.42    & 56.93 $\pm$ 0.75  & 40.18 $\pm$ 0.37    \\
        w/o RDLD         & 68.19 $\pm$ 0.83    & 62.46 $\pm$ 0.24  & 46.79 $\pm$ 0.89    \\
        w/o DAA           & 65.28 $\pm$ 0.61    & 61.23 $\pm$ 0.46  & 45.14 $\pm$ 0.51    \\ 
        ours              & 71.04 $\pm$ 0.35    & 66.02 $\pm$ 0.72  & 52.67 $\pm$ 0.22    \\ \bottomrule
    \end{tabular}  
    \label{table:Ablation study.}
\end{table}


\begin{table}[t]
  \small
  \begin{minipage}[t]{0.46\linewidth} 
    \caption{Effect of sampling rate. Average test accuracy across all clients and all time steps on CIFAR-10. SR is short for sampling rate.}
    \centering
    \resizebox{\linewidth}{!}{
    \begin{tabular}{@{}cccc@{}}
        \toprule
        SR     & F-MNIST       & CIFAR-10      & CIFAR-100    \\ \midrule
        0.2    & 70.56 $\pm$ 0.42    & 65.60 $\pm$ 0.33  & 51.21 $\pm$ 0.85     \\
        0.5    & 71.04 $\pm$ 0.75    & 66.02 $\pm$ 0.56  & 52.67 $\pm$ 0.46     \\
        0.8    & 72.39 $\pm$ 0.37    & 67.42 $\pm$ 0.91  & 53.14 $\pm$ 0.54     \\ 
        1      & 72.86 $\pm$ 0.22    & 67.55 $\pm$ 0.64  & 53.51 $\pm$ 0.39     \\ \bottomrule
    \end{tabular}
    }
    \label{table:Effect of sampling rate.}
  \end{minipage}
  \hfill
  \begin{minipage}[t]{0.46\linewidth} 
    \caption{Effect of dirichlet parameter. Average test accuracy across all clients and all time steps on CIFAR-10. DP is short for dirichlet parameter.}
    \centering
    \resizebox{\linewidth}{!}{
    \begin{tabular}{@{}cccc@{}}
        \toprule
        DP  & F-MNIST       & CIFAR-10      & CIFAR-100    \\ \midrule
        0.5  & 67.94 $\pm$ 0.68    & 63.86 $\pm$ 0.47  & 50.65 $\pm$ 0.91    \\
        1    & 68.19 $\pm$ 0.32    & 64.57 $\pm$ 0.75  & 51.21 $\pm$ 0.38    \\
        2    & 69.56 $\pm$ 0.44    & 65.72 $\pm$ 0.59  & 52.06 $\pm$ 0.49    \\ \bottomrule
    \end{tabular}
    }
    \label{table:Effect of dirichlet parameter.}
  \end{minipage}
\end{table}
\textbf{Comparison with baselines on forgetting mitigation.}
Table~\ref{table:Comparison with baselines in avoiding forgetting.} shows the test accuracy on each task at time $t=6$, reflecting how well each method retains knowledge over time. Our method consistently achieves the best performance across all tasks. In contrast, single-model baselines such as FedAVG and A-FedAvg suffer severe forgetting, with accuracy on earlier tasks often falling below 20\%. Compared to clustered FL methods like FedRC and FedDrift, our approach achieves higher accuracy on both early and recent tasks. This suggests that using different adaptation strategies for real and virtual/label drift helps preserve useful historical knowledge more effectively.




\textbf{Ablation study on different modules of our method.} To evaluate the contribution of each key module in our method, we conduct an ablation study by removing the NCD, RDLD, and DAA modules respectively. The results in Table~\ref{table:Ablation study.} demonstrate that each module contributes positively to the overall performance. Removing the NCD module causes the largest drop in accuracy across all datasets, indicating its critical role in clustering clients effectively. The absence of the RDLD and DAA modules also leads to noticeable performance degradation, highlighting their importance in detecting real drift and adapting to multi-source concept drift.

\textbf{Effect of sampling rate.} Table~\ref{table:Effect of sampling rate.} reports the performance of our method under varying client sampling rates. As the sampling rate decreases from 1 to 0.2, the average accuracy only slightly declines across all datasets. This demonstrates that our method remains effective even when fewer clients participate in training.

\textbf{Effect of dirichlet parameter.} Table~\ref{table:Effect of dirichlet parameter.} presents the impact of varying the Dirichlet parameter, which controls label distribution heterogeneity. As the parameter decreases from 2 to 0.5, clients' data distributions become more heterogeneous, making the learning task more challenging. Despite this, FedDAA maintains strong performance, demonstrating its robustness to label drift.


%% file: sec/7_conclusion.tex
\section{Conclusion}

In this work, we tackle the critical challenges of concept drift in FL. we introduce FedDAA, a dynamic clustered FL framework that adapts to multi-source concept drift while retaining useful historical knowledge. Theoretically, we establish rigorous guarantees on convergence rate under time-varying distributions, advancing the understanding of FL convergence in non-stationary environments. Experimental results demonstrate significant improvements, with our method outperforming state-of-the-art baselines like FedDrift by 7.84\%–8.52\% in accuracy across various datasets. These contributions open new possibilities for deploying robust FL systems in real-world dynamic applications such as IoT networks and personalized healthcare, where data distributions naturally evolve over time. Despite these promising results, our framework currently relies on a relatively simple rehearsal strategy to mitigate forgetting under concept drift. Investigating more advanced continual learning techniques in future work may further enhance the system's robustness and adaptability.

%% file: sec/appendix.tex
\appendix

\section{Algorithm}
\subsection{Number of Clusters Determination}\label{sec:appendix Determine the Cluster Number}
Algorithm~\ref{algo:K-means Based Determine Cluster Number Algorithm} presents the pseudocode of the NCD module at a given time step $t$. Lines~3--6: each client $k$ obtains the prototype $\mathbf{P}_{k}^{t}$. Lines~7--11: The server computes the average silhouette scores for each possible number of clusters. Line~12: Choose the optimal number of clusters that maximizes the average silhouette scores.

\textbf{K-means algorithm.} The K-means algorithm is a widely used clustering method. In our work, we use it to cluster the data prototypes $\mathbf{P}_{k}^{t}$ into \( C^{t} \) distinct clusters at time step $t$. The K-means algorithm iterates between assigning each prototype to its nearest cluster center and recalculating each cluster center as the mean of its assigned prototypes. 

\textbf{Average silhouette score for cluster number determination.} The average silhouette score is employed to assess the quality of clustering results. Silhouette score measures how similar a prototype is to its own cluster compared to other clusters, with values ranging from $-1$ to $1$. A higher silhouette score indicates better clustering performance. By computing the average silhouette scores for different numbers of clusters, we can identify the optimal number of clusters that maximizes the scores. 
To calculate the silhouette score \(s_{k}^{t}\) for a data prototype \(\mathbf{P}_{k}^{t}\), we proceed as follows. First, determine the average intra-cluster distance \(a_{k}^{t}\), which represents the average distance from \(\mathbf{P}_{k}^{t}\) to other prototypes within the same cluster. Next, find the average nearest-cluster distance \(b_{k}^{t}\), defined as the average distance from \(\mathbf{P}_{k}^{t}\) to the prototypes in the nearest neighboring cluster. Finally, the silhouette score $s_{k}^{t}$ is given by
\begin{equation}
s_{k}^{t} = \frac{b_{k}^{t} - a_{k}^{t}}{\max(a_{k}^{t}, b_{k}^{t})}.
\end{equation}

The average silhouette score \(\bar{s}_{c}^{t}\) for \(c\) clusters is then computed using Eq.~\eqref{eq: the average silhouette}.
\begin{equation}\label{eq: the average silhouette}
    \bar{s}_{c}^{t} = \sum_{k=1}^{K} s_{k}^{t}/K
\end{equation}

\begin{algorithm}[ht]
\caption{Number of Clusters Determination (NCD) module }
\label{algo:K-means Based Determine Cluster Number Algorithm}
\begin{algorithmic}[1]
\State \textbf{Input:} Client set $\mathcal{K} = \{1, \dots, K\}$, maximal number of clusters $M$, prototype $\mathbf{P}_{k}^{t}$ of client $k$, client prototypes set $\mathcal{R}_{\text{proto}}^{t}\gets \varnothing$, average silhouette scores~list $ \mathcal{J}^{t} \gets [\,\,]$  and average silhouette score $\bar{s}_{c}^{t}$ for the clustering result with $c$ clusters. 
\State \textbf{Output:} Optimal number of clusters $C^{t}$ and prototype cluster centers for time step $t$ 
\For{each client $k\in \mathcal{K}$}\Comment{Operations on the client side.}
    \State Client $k$ obtain $\mathbf{P}_{k}^{t}$ using Eq.~\eqref{eq:prototype} and send it to the server
    \State $\mathcal{R}_{\text{proto}}^{t} \gets \mathcal{R}_{\text{proto}}^{t} \cup \{\mathbf{P}_{k}^{t}\}$ \Comment{The server adds $\mathbf{P}_{k}^{t}$ to $\mathcal{R}_{\text{proto}}^{t}$.} 
\EndFor

\For{ each number of clusters $c = 2,\ldots,M$}\Comment{Operations on the server side.}
    \State The server groups prototypes in $\mathcal{R}_{\text{proto}}^{t}$ into $c$ clusters and obtain cluster centers using K-means algorithm
    \State Compute the average silhouette score $\bar{s}_{c}^{t}$ using Eq.~\eqref{eq: the average silhouette} based on the clustering result
    \State Append $\bar{s}_{c}^{t}$ to $\mathcal{J}^{t}$  
\EndFor
\State $C^{t} \gets \arg\max(\mathcal{J}^{t})$ \Comment{Select the number of clusters with max the average silhouette scores.}

\State \Return Optimal number of clusters $C^{t}$ and prototype cluster centers for time step $t$
\end{algorithmic}
\end{algorithm}

\subsection{Real Drift Local Detection}\label{sec:appendix Real Drift Detection}
The RDLD module outputs two disjoint client sets, $\mathcal{K}_{\text{drift}}$ and $\mathcal{K}_{\text{clean}}$, such that $\mathcal{K}_{\text{drift}} \cup \mathcal{K}_{\text{clean}} =\mathcal{K}$ and $\mathcal{K}_{\text{drift}} \cap \mathcal{K}_{\text{clean}} = \varnothing$. The set $\mathcal{K}_{\text{drift}}$ contains the clients that are experiencing real drift, whereas the set $\mathcal{K}_{\text{clean}}$ contains the clients without real drift. 
Algorithm~\ref{algo:Prototype based real drift detection method} presents the pseudocode of the RDLD module. Line~1: Initialize two client sets, i.e., $\mathcal{K}_{\text{drift}} \leftarrow \varnothing$ and $\mathcal{K}_{\text{clean}} \leftarrow \varnothing$. Lines~4--5: Each client $k$ obtains the prototypes of previous dataset $S_{k}^{t-1}$ and current dataset $S_{k}^{t}$, respectively. Lines~6--7: Use the \textit{Predict Cluster} function to determine the clusters to which the previous prototype $\mathbf{P}_{k}^{t-1}$ and current prototype $\mathbf{P}_{k}^{t}$ belong.  Lines~8--12: Perform real drift detection. If $c_{k}^{t-1} = c_{k}^{t}$, it indicates that the conditional distribution of client $k$'s data remains unchanged. In this case, client $k$ is added to the set $\mathcal{K}_{\text{clean}}$. Otherwise, the client is experiencing real drift and is therefore added to the set $\mathcal{K}_{\text{drift}}$. Lines~15--22 present the function \textit{Predict Cluster}. Lines~16--19: Compute the distances from the given prototype to all cluster centers for the current time step.  Line 21: Obtain the index of the cluster whose center is closest to the given prototype.

\begin{algorithm}[t]
\caption{Real Drift Local Detection (RDLD) Module }
\label{algo:Prototype based real drift detection method}
\begin{algorithmic}[1]
    \Require Client set $\mathcal{K} = \{1, \dots, K\}$, global models $\textbf{w}^{t-1} = [\boldsymbol{w}_{1}^{t-1},...,\boldsymbol{w}_{C^{t-1}}^{t-1}]$,  prototypes $\mathbf{P}_{k}^{t-1}$ and $\mathbf{P}_{k}^{t}$ of client $k$, and prototype cluster centers set $\mathcal{R}_{\text{centers}}^{t}$. 
    \Ensure Drift set $\mathcal{K}_{\text{drift}}$, clean set $\mathcal{K}_{\text{clean}}$
        \State Initialize $\mathcal{K}_{\text{drift}} \gets \varnothing$, $\mathcal{K}_{\text{clean}} \gets \varnothing$
        \State The server sends prototype cluster centers to each client $k$
        \For{each client $k \in \mathcal{K}$}
            \State Client $k$ obtains $\mathbf{P}_{k}^{t-1}$ of previous dataset $S_{k}^{t-1}$ using Eq.~\eqref{eq:prototype}
            \State Client $k$ obtains $\mathbf{P}_{k}^{t}$ of  current dataset $S_{k}^{t}$ using Eq.~\eqref{eq:prototype}
            \State $c_{k}^{t-1} \gets \text{Predict Cluster}(\mathbf{P}_{k}^{t-1}, \mathcal{R}_{\text{centers}}^{t})$
            \State $c_{k}^{t} \gets \text{Predict Cluster}(\mathbf{P}_{k}^{t}, \mathcal{R}_{\text{centers}}^{t})$
            \If{$c_{k}^{t-1} = c_{k}^{t}$}
                \State $\mathcal{K}_{\text{clean}} \gets \mathcal{K}_{\text{clean}} \cup \{k\}$ \Comment{Client $k$ shows no real drift}
            \Else
                \State $\mathcal{K}_{\text{drift}} \gets \mathcal{K}_{\text{drift}} \cup \{k\}$ \Comment{Client $k$ experiences real drift}
            \EndIf
        \EndFor
        \State \Return $\mathcal{K}_{\text{drift}}, \mathcal{K}_{\text{clean}}$
        
\Procedure{Predict Cluster}{$\mathbf{P}, \mathcal{R}_{\text{centers}}^{t}$}
    \State Initialize distance list $\mathcal{D}_{\text{distance}} \gets [\,\,]$ 
    \For{each cluster center $\mathbf{P}_{c}^{t} \in \mathcal{R}_{\text{centers}}^{t}$}
        \State Compute distance $d(\mathbf{P}, \mathbf{P}_{c}^{t})$ using Eq.~\eqref{eq:Euclidean distance}
        \State Append $d(\mathbf{P}, \mathbf{P}_{c}^{t})$ to $\mathcal{D}_{\text{distance}}$
    \EndFor
    \State $c \gets \arg\min \mathcal{D}_{\text{distance}}$ \Comment{Index of the closest cluster center}
    \State \Return $c$
\EndProcedure
    \end{algorithmic}
\end{algorithm}

\subsection{FedDAA} \label{sec:appendix Concept Drift Adaptation}
Algorithm~\ref{algo:Adapatation Algorithm} presents the pseudocode for FedDAA. 
Line~4: Determine the number of clusters using Algorithm~\ref{algo:K-means Based Determine Cluster Number Algorithm}. Line~5: Detects clients that have experienced real drift using Algorithm~\ref{algo:Prototype based real drift detection method}. Lines~10--14: If client \(k\) experiences real drift, it uses only the current dataset \(S_{k}^{t}\) for concept drift adaptation. Otherwise, client \(k\) uses both the current dataset \(S_{k}^{t}\) and the previous dataset \(S_{k}^{t - 1}\) for adaptation.  Lines~15--16: After local training, the local models and clustering weights is updated for all clients \(k\), whose optimization method can be found in~\cite{DBLP:conf/icml/GuoTL24}. Clients send their local models to server. Line~18: The global model parameters \(\boldsymbol{w}_{c}^{j + 1}\) are updated. Line~20: After \(\tau\) communication rounds, the model  \(\boldsymbol{w}_{c}^{t}\) at time step $t$ are set to \(\boldsymbol{w}_{c}^{\tau}\) for all \(c\). 

\begin{algorithm}[t]
    \caption{FL Distribution-Awared Adaptation (FedDAA) Algorithm}
    \label{algo:Adapatation Algorithm}
    \begin{algorithmic}[1]
    \State \textbf{Input:} Global models  $\mathbf{w}^{0} = [\boldsymbol{w}^{0}_1, \ldots, \boldsymbol{w}^{0}_{C^{0}}]$, global learning rate $\eta_g$, local learning rate $\eta_l$, and datasets $S_{k}^{t}$ for each client $k$.

    \State \textbf{Output:} New global models $\mathbf{w}^T = [\boldsymbol{w}^T_1, \ldots, \boldsymbol{w}^T_{C^{T}}]$ and weights $\hat{\alpha}^T_{k,c}$ for any $k\in \mathcal{K}, c \in \mathcal{C}^{T}$.
    \For{time step $t=1, \ldots, T$}

        \State Obtain the number of clusters $C^t$ using Algorithm~\ref{algo:K-means Based Determine Cluster Number Algorithm}
        \State Detect real drift for each client using Algorithm~\ref{algo:Prototype based real drift detection method}, and divide them into the clean set $\mathcal{K}_{\text{clean}}$ and the drift set $\mathcal{K}_{\text{drift}}$ accordingly.

        \For{communication round $j = 1, \ldots, \tau$}
            \State Server sends $\boldsymbol{w}_{c}^{j}$ to the chosen clients
            \For{client $k$ in parallel} \Comment{Concept drift adaptation.}
                \State Initialize local model $\boldsymbol{w}_{k,c}^{j} = \boldsymbol{w}_c^{j}$, $\forall c$
                \If{$k \in \mathcal{K}_{\text{drift}}$}
                \State  Client $k$ will use  $S_{k}^{t}$ to perform local update \Comment{ Client $k$ experiences real drift.}
                \Else
                \State  Client $k$ will use  both $S_{k}^{t}$ and $S_{k}^{t-1}$  to perform local update
                \EndIf
                \State Client $k$ updates $\hat{\alpha}_{k,c}^{j}$ and  $\boldsymbol{w}_{k,c}^{j}$ by Eq.~\eqref{eq:our alpha} and~\eqref{eq:our w}, $\forall c$
                \State Client $k$ sends $\boldsymbol{w}_{k,c}^{j}$, $\forall c$ to server
            \EndFor
            \State $\boldsymbol{w}_{c}^{j+1} \leftarrow \frac{\eta_g}{\sum_{k} N_k^{t}} \sum_{k} N_k^{t}  \boldsymbol{w}_{k,c}^{j}$, $\forall c$
        \EndFor
        \State $\boldsymbol{w}_{c}^{t}\leftarrow \boldsymbol{w}_{c}^{\tau}$, $\forall c \in \mathcal{C}^{t}$
        \State $\mathbf{w}^{t}\leftarrow [\boldsymbol{w}^t_{1}, \ldots, \boldsymbol{w}^t_{C^{t}}]$ 
    \EndFor
    \end{algorithmic}
\end{algorithm}

\section{Proofs of Convergence Analysis}\label{appendix: Proofs of convergence analysis}
\subsection{Approximated objective function for piratical implementation}\label{appendix: Proofs of objective funtion}
To optimize Eq.~\eqref{eq:objective funtion}, an approximated objective function is needed for practical implementation~\cite{DBLP:conf/icml/GuoTL24}. Note that $\mathcal{I}(\mathbf{x}^t, y^t; \boldsymbol{w}_c^t)$ cannot be directly evaluated in practice. To simplify the implementation,  $\mathcal{I}(\mathbf{x}^t, y^t; \boldsymbol{w}_c^t)$ is calculated by $\frac{\mathcal{P}(y^t|\mathbf{x}^t; \boldsymbol{w}_c^t)}{\mathcal{P}(y^t;\boldsymbol{w}_c^t)}$, and both $\mathcal{P}(y^t|\mathbf{x}^t; \boldsymbol{w}_c^t)$ and $\mathcal{P}(y^t;\boldsymbol{w}_c^t)$ need to be approximated. Therefore, $\widetilde{\mathcal{I}}(\mathbf{x}^t, y^t; \boldsymbol{w}_c^t)$ is introduced as an approximation of $\mathcal{I}(\mathbf{x}^t, y^t; \boldsymbol{w}_c^t)$, and the refined definition of Eq.\eqref{eq:objective funtion} is shown below:

\begin{equation}
\mathcal{L}(\mathbf{w}^t, \boldsymbol{\hat{\alpha}}^t) = \frac{1}{N} \sum_{k=1}^{K} \sum_{i=1}^{N_k^t} \ln \left( \sum_{c=1}^{C^t} \hat{\alpha}_{k,c}^t \widetilde{\mathcal{I}}(\mathbf{x}_{k,i}^t, y_{k,i}^t, \boldsymbol{w}_c^t) \right),
\label{eq:our approximated objective}
\end{equation}

s.t. $\sum_{c=1}^{C^t} \hat{\alpha}_{k,c}^t = 1, \forall k$,
where $\widetilde{\mathcal{I}}(\mathbf{x}^t, y^t; \boldsymbol{w}_c^t) = \frac{\exp(-f(\mathbf{x}^t, y^t, \boldsymbol{w}_c^t))}{A_{y,c}^t}$. Note that $f(\mathbf{x}, y, \boldsymbol{w}_k)$ is the loss function defined by $-\ln \mathcal{P}(y|\mathbf{x}; \boldsymbol{w}_c) + o$ for some constant $o$~\cite{DBLP:conf/nips/MarfoqNBKV21}. $A_{y,c}^t$ is the constant that used to approximate $\mathcal{P}(y^t; \boldsymbol{w}_c^t)$ in practice.

The intuition behind using $\widetilde{\mathcal{I}}(\mathbf{x}^t, y^t; \boldsymbol{w}_c^t)$ as an approximation comes from:

\begin{itemize}
    \item The fact of $f(\mathbf{x}^{t}, y^{t}, \boldsymbol{w}_c^{t}) \propto -\ln \mathcal{P}(y^{t}|\mathbf{x}^{t}; \boldsymbol{w}_c^{t})$. $\mathcal{P}(y^{t}|\mathbf{x}^{t}; \boldsymbol{w}_c^{t})$ can be approximated using $\exp(-f(\mathbf{x}^{t}, y^{t}, \boldsymbol{w}_c^{t}))$. 
    \item $A_{y,c}^{t}$ is calculated by
    $\frac{\frac{1}{N^{t}} \sum_{k=1}^{K} \sum_{i=1}^{N_k^{t}} \mathbf{1}_{\{y_{k,i}=y\}} \gamma_{k,i;c}^{t}}{\frac{1}{N^{t}} \sum_{k=1}^{K} \sum_{i=1}^{N_k^{t}} \gamma_{k,i;c}^{t}},
    $
    where $\gamma_{k,i;c}^{t}$ represents the weight of data $(\mathbf{x}_{k,i}^{t}, y_{k,i}^{t})$ assigned to $\boldsymbol{w}_c^{t}$. Thus, $A_{y,c}^{t}$ corresponds to the proportion of data pairs labeled as $y^{t}$ that choose model $\boldsymbol{w}_c^{t}$, and can be used to approximate $\mathcal{P}(y^{t}; \boldsymbol{w}_c^{t})$.
\end{itemize}

\begin{lemma}[Optimization for Eq.~\eqref{eq:our approximated objective}, Theorem A.1 in ~\cite{DBLP:conf/icml/GuoTL24}]\label{lemma:our optimization of objective function} Let $j \in \{0, 1, 2, \ldots, \tau-1\}$ denote the $j$-th communication round at time step $t$.
When maximizing Eq.~\eqref{eq:our approximated objective}, the optimization steps are given by,

Optimizing $\boldsymbol{\hat{\alpha}}$:
\begin{align}
\gamma_{k,i;c}^{j} &= \frac{\hat{\alpha}_{k,c}^{j-1} \widetilde{\mathcal{I}}(\mathbf{x}_{k,i}^{t}, y_{k,i}^{t}, \boldsymbol{w}_c)}{\sum_{k=1}^K \hat{\alpha}_{k,c}^{j-1} \widetilde{\mathcal{I}}(\mathbf{x}_{k,i}^{t}, y_{k,i}^{t}, \boldsymbol{w}_c)},\label{eq:our gamma} \\
\hat{\alpha}_{k,c}^{j} &= \frac{1}{N_k^{t}} \sum_{i=1}^{N_k^{t}} \gamma_{k,i;c}^{j}.\label{eq:our alpha}
\end{align}

Optimizing $\mathbf{w}$:
\begin{align}
\boldsymbol{w}_c^{j} = \boldsymbol{w}_c^{j-1} - \frac{\eta}{N^{t}} \sum_{k=1}^K \sum_{i=1}^{N_k^{t}} \gamma_{k,i;c}^{j} 
\nabla_{\boldsymbol{w}_c^{j-1}} f(\mathbf{x}_{k,i}^{t}, y_{k,i}^{t}, \boldsymbol{w}_c^{j-1}), \label{eq:our w}
\end{align}
where $\eta$ is learning rate.
\end{lemma}
Lemma~\ref{lemma:our optimization of objective function} can be proven easily using lagrange multiplier method.

\subsection{Convergence Analysis for Centralized Clustering Method}
In this section, we present the convergence analysis of the centralized version of FedDAA. At each time step $t$, $\tau$ communication rounds are used to retrain the models and adapt them to concept drift. In the following analysis, we focus on the training process within a single time step and omit the superscript $t$ for notational simplicity.
\begin{lemma}[Convergence Analysis of Centralized Clustering Method]
    Assume $ f $ satisfy Assumption~\ref{assumption: our 1}-~\ref{assumption: our 2}, setting $ \tau $ as the number of communication rounds for time step $t$, $ \eta = \frac{8}{40L + 9\sigma^2}$ , we have,
    \begin{equation}
        \begin{aligned}
        &\mathbb{E}_{C}[\mathcal{L}(\mathbf{w}^{\tau}, \boldsymbol{\hat{\alpha}}^{\tau},C^{*}) - \mathcal{L}(\mathbf{w}^{\tau}, \boldsymbol{\hat{\alpha}}^{\tau },C)]\leq \mathcal{L}(\mathbf{w}^{*}, \boldsymbol{\hat{\alpha}}^{*},C^{*}) \\
        &- \mathcal{L}(\mathbf{w}^{0}, \boldsymbol{\hat{\alpha}}^{0},C^{*})-\frac{40L+9\sigma^{2}}{16}\mathbb{E}_{C}[\sum_{c=1}^{C}\|\boldsymbol{w}^{*}_{c}-\boldsymbol{w}^{0}_{c}\|^2],
        \end{aligned}
    \end{equation}
where $C^{*}$ represents the optimal number of clusters.
\label{theorem:our centralized proof}
\end{lemma}

\begin{proof}
Based on the theoretical framework established in Theorem~$4.3$ of FedRC~\cite{DBLP:conf/icml/GuoTL24}, combined with the foundational assumptions \ref{assumption: our 1}--\ref{assumption: our 2}, for a specific cluster numebr $C$ we can get the following inequality:

    \begin{equation}
        \begin{aligned}
            &\mathcal{L}(\mathbf{w}^{\tau }, \boldsymbol{\hat{\alpha}}^{\tau },C) - \mathcal{L}(\mathbf{w}^{0}, \boldsymbol{\hat{\alpha}}^{0},C)\\
            &= \sum_{j=0}^{\tau -1} \mathcal{L}(\mathbf{w}^{j+1}, \boldsymbol{\hat{\alpha}}^{j+1},C) - \mathcal{L}(\mathbf{w}^{j}, \boldsymbol{\hat{\alpha}}^{j},C)\\
            &\geq \frac{\eta}{2}\sum_{j=0}^{\tau -1}\sum_{c=1}^{C}\|\nabla_{\boldsymbol{w}_{c}}\mathcal{L}(\mathbf{w}^{j}, \boldsymbol{\hat{\alpha}}^{j},C)\|^{2}\\
            &= \frac{\eta}{2}\sum_{c=1}^{C}\sum_{j=0}^{\tau -1}\|\nabla_{\boldsymbol{w}_{c}}\mathcal{L}(\mathbf{w}^{j},\boldsymbol{\hat{\alpha}}^{j},C)\|^{2},
        \end{aligned}
    \end{equation}
    where $ \eta = \frac{8}{40L + 9\sigma^2}$, $j \in \{0,1, 2, \ldots, \tau-1\}$ represents the $j$-th communication round at time step $t$.
    
    Because
    \begin{equation}
        \begin{aligned}
        \|\boldsymbol{w}_{c}^{*}-\boldsymbol{w}_{c}^{0}\|^{2} &\approx \|\sum_{j=0}^{\tau-1}\eta\nabla_{\boldsymbol{w}_{c}}\mathcal{L}(\mathbf{w}^{j}, \boldsymbol{\hat{\alpha}}^{j},C)\|^{2}\\
        &\leq \sum_{j=0}^{\tau-1}\eta^{2}\|\nabla_{\boldsymbol{w}_{c}}\mathcal{L}(\mathbf{w}^{j}, \boldsymbol{\hat{\alpha}}^{j},C)\|^{2},
        \end{aligned}
    \end{equation}
    we have
        \begin{equation}
        \begin{aligned}
            \mathcal{L}(\mathbf{w}^{\tau}, \boldsymbol{\hat{\alpha}}^{\tau},C) - \mathcal{L}(\mathbf{w}^{0}, \boldsymbol{\hat{\alpha}}^{0},C) \geq \frac{1}{2\eta}\sum_{c=1}^{C}\|\boldsymbol{w}_{c}^{*}-\boldsymbol{w}_{c}^{0}\|^{2}.
        \end{aligned}
    \end{equation}
    Then we can derive
        \begin{align}
            &\mathbb{E}_{C}[\mathcal{L}(\mathbf{w}^{\tau}, \boldsymbol{\hat{\alpha}}^{\tau},C^{*}) - \mathcal{L}(\mathbf{w}^{\tau}, \boldsymbol{\hat{\alpha}}^{\tau},C)]\notag\\
            &\leq \mathbb{E}_{C}[\mathcal{L}(\mathbf{w}^{\tau}, \boldsymbol{\hat{\alpha}}^{\tau},C^{*}) - \mathcal{L}(\mathbf{w}^{0}, \boldsymbol{\hat{\alpha}}^{0},C)-\frac{1}{2\eta}\sum_{c=1}^{C}\|\boldsymbol{w}_{c}^{*}-\boldsymbol{w}_{c}^{0}\|^{2}]\\
            &\leq \mathcal{L}(\mathbf{w}^{*}, \boldsymbol{\hat{\alpha}}^{*},C^{*}) - \mathcal{L}(\mathbf{w}^{0}, \boldsymbol{\hat{\alpha}}^{0},C^{*})-\frac{1}{2\eta}\mathbb{E}_{C}[\sum_{c=1}^{C}\|\boldsymbol{w}_{c}^{*}-\boldsymbol{w}_{c}^{0}\|^{2}]\label{eq:our centralized proof ineq}.
        \end{align}

    The inequality~\eqref{eq:our centralized proof ineq} holds due to the following two reasons: (i) $\mathbb{E}_{C}[\mathcal{L}(\mathbf{w}^{\tau}, \boldsymbol{\hat{\alpha}}^{\tau},C^{*})]=\mathcal{L}(\mathbf{w}^{\tau}, \boldsymbol{\hat{\alpha}}^{\tau},C^{*})\leq \mathcal{L}(\mathbf{w}^{*}, \boldsymbol{\hat{\alpha}}^{*},C^{*})$; (ii) Under Assumption~\ref{assumption: our 3}, we obtain $\mathbb{E}_{C}[\mathcal{L}(\mathbf{w}^{0}, \boldsymbol{\hat{\alpha}}^{0},C)]=\mathcal{L}(\mathbf{w}^{0}, \boldsymbol{\hat{\alpha}}^{0},C^{*})$.
    
\end{proof}



\begin{corollary}\label{corollary: our centralized proof}
Under Assumption~\ref{assumption: our 1},\ref{assumption: our 2} and \ref{assumption: our 3}, let $ \tau $ as the number of communication rounds, $ \eta = \frac{8}{40L + 9\sigma^2} $, $\Delta_{c}^{2} = \mathbb{E}[\|\boldsymbol{w}_{c}^{*}-\boldsymbol{w}_{c}^{0}\|^{2}] $ for all $c$, and $\Delta^2 = \max \{\Delta_{c}^{2},c=1,2,3,...,C.\}$. we have,
\begin{equation}
    \begin{aligned}
    &\mathbb{E}_{C}[\mathcal{L}(\mathbf{w}^{\tau}, \boldsymbol{\hat{\alpha}}^{\tau},C^{*}) - \mathcal{L}(\mathbf{w}^{\tau}, \boldsymbol{\hat{\alpha}}^{\tau},C)]\\
    &\leq \mathcal{L}(\mathbf{w}^{*}, \boldsymbol{\hat{\alpha}}^{*},C^{*}) - \mathcal{L}(\mathbf{w}^{0}, \boldsymbol{\hat{\alpha}}^{0},C^{*})-\frac{40L+9\sigma^{2}}{16}C^{*}\Delta^2.
    \end{aligned}
\end{equation}
\end{corollary}

\begin{proof}
    Applying the law of total expectation(also known as the tower rule), we obtain: $\mathbb{E}_{C}[\sum_{c=1}^{C}\|\boldsymbol{w}^{\tau}_{c}-\boldsymbol{w}^{0}_{c}\|^2]=\mathbb{E}_C[\mathbb{E}_{\boldsymbol{w}}[\sum_{c=1}^{C}\|\mathbf{w}^{\tau}_{c}-\boldsymbol{w}^{0}_{c}\|^2|C]]$.
    Applying Assumption~\ref{assumption: our 3}, we have
    \begin{equation}
        \begin{aligned}
            \mathbb{E}_{C}[\sum_{c=1}^{C}\|\boldsymbol{w}^{\tau}_{c}-\boldsymbol{w}^{0}_{c}\|^2]&=\mathbb{E}_{C}[\mathbb{E}_{\boldsymbol{w}}[\sum_{c=1}^{C}\|\boldsymbol{w}^{\tau}_{c}-\boldsymbol{w}^{0}_{c}\|^2|C]]\\
            &=\mathbb{E}_{C}[\sum_{c=1}^{C}\mathbb{E}_{\boldsymbol{w}}[\|\boldsymbol{w}^{\tau}_{c}-\boldsymbol{w}^{0}_{c}\|^2]]\\
            &=\mathbb{E}_{C}[\sum_{c=1}^{C}\Delta^2]\\
            &=\Delta^2\mathbb{E}{C}\\
            &\leq C^{*}\Delta^2
        \end{aligned}
    \end{equation}
    Finally, applying Lemma~\ref{theorem:our centralized proof}, we can complete the proof.
\end{proof}

\subsection{Convergence Analysis for Federated Clustering Method}\label{sec:appendix Convergence Analysis for Federated Clustering Method}


\setcounter{theorem}{0} 
\begin{theorem} [Convergence Analysis of FedDAA]
     Under Assumptions~\ref{assumption: our 1}, \ref{assumption: our 2}, \ref{assumption: our 5}, \ref{assumption: our 6} and \ref{assumption: our 3}, clients use SGD as local solver with learning rate $\eta \leq \frac{1}{\sqrt{\tau}}$, and run our method for $\tau$ communication rounds and $E$ local epochs for each local update. Let $ \Delta_{\xi,c}^{2} = \mathbb{E}_{\xi}\|\boldsymbol{w}_{c}^{\tau}-\boldsymbol{w}_{c}^{0}\|^{2}$ for all $c$, and $\Delta_{\xi}^2 = \max \{\Delta_{\xi,c}^{2},c=1,2,3,...,C.\}$. We have
\begin{equation}
    \mathbb{E}_{C}\mathbb{E}_{\xi}[\mathcal{L}(\mathbf{w}^{\tau}, \boldsymbol{\hat{\alpha}}^{\tau},C^{*}) - \mathcal{L}(\mathbf{w}^{\tau}, \boldsymbol{\hat{\alpha}}^{\tau},C)] \leq \mathcal{L}(\mathbf{w}^{*}, \boldsymbol{\hat{\alpha}}^{*},C^{*}) -\mathcal{L}(\mathbf{w}^{0}, \boldsymbol{\hat{\alpha}}^{0},C^{*})+B,
\end{equation}
where 
\begin{equation}
    B= -\frac{EC^{*}\Delta_{\xi}^{2}}{4\eta}+ 2(E\eta^2L(\eta L+1)\sigma^2+2E^{3}\eta^{3}L^{2}G^{2}).
\end{equation}
\end{theorem}

\begin{proof}
    For a client $k$ and its data $(\mathbf{x}_{k,i}, y_{k,i})$, $i=1,2,...,N_{k}$, constructing

\begin{equation}
\begin{aligned}
&g_k^j(\mathbf{w}, \boldsymbol{\hat{\alpha}},C) = \frac{1}{N_k} \sum_{i=1}^{N_k} \sum_{c=1}^{C} \gamma_{k,i;c}^j \left( f(\mathbf{x}_{k,i}, y_{k,i}, \boldsymbol{w}_c)\right.\\
&\left.+ \log(\mathcal{P}_c(y)) - \log(\hat{\alpha}_{k,c}) + \log(\gamma_{k,i;c}^j) \right),
\end{aligned}
\end{equation}
where $\mathbf{w}=[\boldsymbol{w}_1,\boldsymbol{w}_2,...,\boldsymbol{w}_{C}]$.

Then we would like to show the following four conditions are satisfied.
\begin{enumerate}[label=(\arabic*)]
    \item $g_k^j(\mathbf{w}, \boldsymbol{\hat{\alpha}},C)$ is L-smooth to $\mathbf{w}$,
    \item $g_k^j(\mathbf{w}, \boldsymbol{\hat{\alpha}},C) \geq -\mathcal{L}_k(\mathbf{w}, \boldsymbol{\hat{\alpha}},C)$,
    \item $g_k^j(\mathbf{w}, \boldsymbol{\hat{\alpha}},C)$ and $-\mathcal{L}_k(\mathbf{w}, \boldsymbol{\hat{\alpha}},C)$ have the same gradient on $\mathbf{w}$,
    \item $g_k^j(\mathbf{w}^{j-1}, \boldsymbol{\hat{\alpha}}^{j-1},C) = \mathcal{L}_k(\mathbf{w}^{j-1}, \boldsymbol{\hat{\alpha}}^{j-1},C)$.
\end{enumerate}

Firstly, we prove condition (1). it is obviously that $g_k^j(\mathbf{w}, \boldsymbol{\hat{\alpha}},C)$ is L-smooth to $\mathbf{w}$ as it is a linear combination of $C$ smooth functions.

Secondly, we prove condition (2). We define $ r(\mathbf{w}, \boldsymbol{\hat{\alpha}},C) = g_k^t(\mathbf{w}, \boldsymbol{\hat{\alpha}},C) + \mathcal{L}_k(\mathbf{w}, \boldsymbol{\hat{\alpha}},C) $, we will have
\begin{equation}
    \begin{aligned}
    &r(\mathbf{w}, \boldsymbol{\hat{\alpha}},C) = g_k^t(\mathbf{w}, \boldsymbol{\hat{\alpha}},C) + \mathcal{L}_k(\mathbf{w}, \boldsymbol{\hat{\alpha}},C)  \\
    &= \frac{1}{N_k} \sum_{i=1}^{N_k} \sum_{c=1}^{C} \gamma_{k,i;c}^j \left( \log(\gamma_{k,i;c}^j) - \log(\hat{\alpha}_{k,c} \widetilde{\mathcal{I}}(\mathbf{x}, y, \boldsymbol{w}_c)) \right)+ \mathcal{L}_k(\mathbf{w}, \boldsymbol{\hat{\alpha}},C) \\
    &= \frac{1}{N_k} \sum_{i=1}^{N_k} \sum_{c=1}^{C} \gamma_{k,i;c}^j ( \log(\gamma_{k,i;c}^j) - \log(\hat{\alpha}_{k,c} \widetilde{\mathcal{I}}(\mathbf{x}, y, \boldsymbol{w}_c)) + \log(\sum_{n=1}^{C} \hat{\alpha}_{k,n} \widetilde{\mathcal{I}}(\mathbf{x}, y, \boldsymbol{w}_n)) )\\ 
    &= \frac{1}{N_k} \sum_{i=1}^{N_k} \sum_{c=1}^{C} \gamma_{k,i;c}^j \left( \log(\gamma_{k,i;c}^j)- \log \left( \frac{\hat{\alpha}_{k,c} \widetilde{\mathcal{I}}(\mathbf{x}, y, \boldsymbol{w}_c)}{\sum_{n=1}^{C} \hat{\alpha}_{k,n} \widetilde{\mathcal{I}}(\mathbf{x}, y, \boldsymbol{w}_n)} \right) \right)\\
    &= \frac{1}{N_k} \sum_{i=1}^{N_k} KL \left( \gamma_{k,i;c}^j \| \frac{\hat{\alpha}_{k,c} \widetilde{\mathcal{I}}(\mathbf{x}, y, \boldsymbol{w}_c)}{\sum_{n=1}^{C} \hat{\alpha}_{k,n} \widetilde{\mathcal{I}}(\mathbf{x}, y, \boldsymbol{w}_n)}\right) \geq 0,
    \end{aligned}
\end{equation}
where $KL(\cdot\|\cdot)$ represents the KL divergence.

Thirdly, we prove condition (3). We can easily obtain that

$$
\frac{\partial g_k^t(\mathbf{w}, \boldsymbol{\hat{\alpha}},C)}{\partial \boldsymbol{w}_c} = \frac{1}{N_k} \sum_{i=1}^{N_k} \gamma_{k,i;c}^j 
\nabla_{\boldsymbol{w}_c} f(\mathbf{x}_{k,i}, y_{k,i}, \boldsymbol{w}_c).
$$
This is align with the gradient of $ -\mathcal{L}_k(\mathbf{w}, \boldsymbol{\hat{\alpha}},C) $.

Finally, we prove condition (4). From Eq.~\eqref{eq:our gamma} in Lemma~\ref{lemma:our optimization of objective function}, we can found that
$$ \gamma_{k,i;c}^j = \frac{\hat{\alpha}_{k,c}^{j-1} \widetilde{\mathcal{I}}(\mathbf{x}, y, \boldsymbol{w}_c^{j-1})}{\sum_{n=1}^{C} \hat{\alpha}_{k,n}^{j-1} \widetilde{\mathcal{I}}(\mathbf{x}, y, \boldsymbol{w}_n^{j-1})}.$$ 

Based on the theoretical federated surrogate optimization framework established in Theorem~$3.2^{'}$ of FedEM~\cite{DBLP:conf/nips/MarfoqNBKV21}, and since the above conditions have been verified, we have 

\begin{equation}
\begin{aligned}
\frac{1}{\tau} \sum_{j=1}^{\tau} \sum_{c=1}^{C}\mathbb{E}_{\xi} \left\| \nabla_{\boldsymbol{w}_c} \mathcal{L}(\mathbf{w}^{j}, \boldsymbol{\hat{\alpha}}^{j},C) \right\|^2 &\leq 4 \mathbb{E}_{\xi} \left[ \frac{\mathcal{L}(\mathbf{w}^{*}, \boldsymbol{\hat{\alpha}}^{*},C) - \mathcal{L}(\mathbf{w}^{0}, \boldsymbol{\hat{\alpha}}^{0},C)}{E \eta \tau} \right] \\
&\quad+ 8 \frac{\eta L (\eta L + 1) \cdot \sigma^2 + 2E^2 \eta^2 L^2 G^2}{\tau},
\end{aligned}
\end{equation}
where $E$ denotes the number of local update steps performed by each client. 

Then we have
\begin{align}
&\mathbb{E}_{\xi} \left[ \mathcal{L}(\mathbf{w}^{*}, \boldsymbol{\hat{\alpha}}^{*},C) \right] \notag\\
&\geq  \frac{E\eta}{4} \sum_{j=1}^{\tau} \sum_{c=1}^{C}\mathbb{E}_{\xi} \left\| \nabla_{\boldsymbol{w}_c} \mathcal{L}(\mathbf{w}^{j}, \boldsymbol{\hat{\alpha}}^{j},C) \right\|^2 -2 (E\eta^{2} L (\eta L + 1) \cdot \sigma^2 + 2E^3 \eta^3 L^2 G^2)\notag\\
&\quad+\mathbb{E}_{\xi} \left[  \mathcal{L}(\mathbf{w}^{0}, \boldsymbol{\hat{\alpha}}^{0},C) \right]\\
&=\frac{E\eta}{4}  \sum_{c=1}^{C}\sum_{j=1}^{\tau}\frac{1}{\eta^{2}}\mathbb{E}_{\xi} \left\| \boldsymbol{w}_{c}^{j}-\boldsymbol{w}_{c}^{j-1} \right\|^2 -2 (E\eta^{2} L (\eta L + 1) \cdot \sigma^2+ 2E^3 \eta^3 L^2 G^2)\notag\\
&\quad+\mathbb{E}_{\xi} \left[  \mathcal{L}(\mathbf{w}^{0}, \boldsymbol{\hat{\alpha}}^{0},C) \right]\\
&\geq\frac{E}{4\eta}  \sum_{c=1}^{C}\mathbb{E}_{\xi} \left\| \boldsymbol{w}_{c}^{*}-\boldsymbol{w}_{c}^{0} \right\|^2 -2 (E\eta^{2} L (\eta L + 1) \cdot \sigma^2 + 2E^3 \eta^3 L^2 G^2)+\mathbb{E}_{\xi} \left[  \mathcal{L}(\mathbf{w}^{0}, \boldsymbol{\hat{\alpha}}^{0},C) \right] \label{eq:proof of FedDAA middle}
\end{align}

By applying the inequality~\eqref{eq:proof of FedDAA middle}, we have

\begin{align}
    &\mathbb{E}_{C}\mathbb{E}_{\xi}[\mathcal{L}(\mathbf{w}^{\tau}, \boldsymbol{\hat{\alpha}}^{\tau},C^{*}) - \mathcal{L}(\mathbf{w}^{\tau}, \boldsymbol{\hat{\alpha}}^{\tau},C)] \notag\\
    &\leq \mathbb{E}_{C}[\mathcal{L}(\mathbf{w}^{\tau}, \boldsymbol{\hat{\alpha}}^{\tau},C^{*}) - \frac{E}{4\eta}  \sum_{c=1}^{C}\mathbb{E}_{\xi} \left\| \boldsymbol{w}_{c}^{*}-\boldsymbol{w}_{c}^{0} \right\|^2 +2 (E\eta^{2} L (\eta L + 1)  \sigma^2 + 2E^3 \eta^3 L^2 G^2) \notag\\
    &\quad -\mathbb{E}_{\xi} [  \mathcal{L}(\mathbf{w}^{0}, \boldsymbol{\hat{\alpha}}^{0},C) ] \label{eq:ineq1}\\
    &\leq \mathcal{L}(\mathbf{w}^{*}, \boldsymbol{\hat{\alpha}}^{*},C^{*})+2 (E\eta^{2} L (\eta L + 1) \cdot \sigma^2 + 2E^3 \eta^3 L^2 G^2) \notag\\
    &\quad -\mathbb{E}_{C}\left[\frac{E}{4\eta}  \sum_{c=1}^{C}\mathbb{E}_{\xi} \left\| \boldsymbol{w}_{c}^{*}-\boldsymbol{w}_{c}^{0} \right\|^2 + \mathbb{E}_{\xi} \left[  \mathcal{L}(\mathbf{w}^{0}, \boldsymbol{\hat{\alpha}}^{0},C) \right] \right] \label{eq:proof of FedDAA ineq2}\\
    &= \mathcal{L}(\mathbf{w}^{*}, \boldsymbol{\hat{\alpha}}^{*},C^{*}) -  \mathcal{L}(\mathbf{w}^{0}, \boldsymbol{\hat{\alpha}}^{0},C^{*}) +2 (E\eta^{2} L (\eta L + 1) \sigma^2+ 2E^3 \eta^3 L^2 G^2) \notag\\
    &\quad - \frac{E}{4\eta} \mathbb{E}_{C}\left[ \sum_{c=1}^{C}\mathbb{E}_{\xi} \left\| \boldsymbol{w}_{c}^{*}-\boldsymbol{w}_{c}^{0} \right\|^2 \right] \label{eq:proof of FedDAA ineq3}
\end{align}

To justify inequality~\eqref{eq:proof of FedDAA ineq2}, we note that the randomness in $C$ does not affect $C^*$, hence $\mathbb{E}_{C}[\mathcal{L}(\mathbf{w}^{\tau}, \boldsymbol{\hat{\alpha}}^{\tau},C^{*})] = \mathcal{L}(\mathbf{w}^{\tau}, \boldsymbol{\hat{\alpha}}^{\tau},C^{*})$. Furthermore, since $(\mathbf{w}^*, \boldsymbol{\hat{\alpha}}^*)$ is optimal under $C^*$, we have $\mathcal{L}(\mathbf{w}^{\tau}, \boldsymbol{\hat{\alpha}}^{\tau},C^{*}) \leq \mathcal{L}(\mathbf{w}^{*}, \boldsymbol{\hat{\alpha}}^{*},C^{*})$.
The inequality~\eqref{eq:proof of FedDAA ineq3} holds due to the following two reasons: (i) Since $\xi $ is independent of $\mathbf{w}^{0}$ and $ \boldsymbol{\hat{\alpha}}^{0}$, we have $\mathbb{E}_{\xi} [  \mathcal{L}(\mathbf{w}^{0}, \boldsymbol{\hat{\alpha}}^{0},C) ]=\mathcal{L}(\mathbf{w}^{0}, \boldsymbol{\hat{\alpha}}^{0},C)$;  (ii) under Assumption~\ref{assumption: our 3}, we obtain $\mathbb{E}_{C}[\mathcal{L}(\mathbf{w}^{0}, \boldsymbol{\hat{\alpha}}^{0},C)]=\mathcal{L}(\mathbf{w}^{0}, \boldsymbol{\hat{\alpha}}^{0},C^{*})$.

By the tower property of conditional expectation, we have $\mathbb{E}_{C}[\sum_{c=1}^{C}\mathbb{E}_{\xi}\|\boldsymbol{w}^{\tau}_{c}-\boldsymbol{w}^{0}_{c}\|^2]=\mathbb{E}_C[\mathbb{E}_{\boldsymbol{w}}[\sum_{c=1}^{C}\mathbb{E}_{\xi}\|\mathbf{w}^{\tau}_{c}-\boldsymbol{w}^{0}_{c}\|^2|C]]$.
Applying Assumption~\ref{assumption: our 3}, we have
    \begin{align}
        \mathbb{E}_{C}[\sum_{c=1}^{C}\mathbb{E}_{\xi}\|\boldsymbol{w}^{\tau}_{c}-\boldsymbol{w}^{0}_{c}\|^2]&=\mathbb{E}_{C}[\mathbb{E}_{\boldsymbol{w}}[\sum_{c=1}^{C}\mathbb{E}_{\xi}\|\boldsymbol{w}^{\tau}_{c}-\boldsymbol{w}^{0}_{c}\|^2|C]]\\
        &=\mathbb{E}_{C}[\sum_{c=1}^{C}\mathbb{E}_{\boldsymbol{w}}[\mathbb{E}_{\xi}\|\boldsymbol{w}^{\tau}_{c}-\boldsymbol{w}^{0}_{c}\|^2]]\\
        &=\mathbb{E}_{C}[\sum_{c=1}^{C}\Delta_{\xi,c}^2]\\
        &=\Delta_{\xi,c}^2\mathbb{E}{C}\\
        &\leq C^{*}\Delta_{\xi}^2 \label{eq:proof of FedDAA delta ineq}
    \end{align}
By applying inequality~\eqref{eq:proof of FedDAA delta ineq} to inequality~\eqref{eq:proof of FedDAA ineq3}, we obtain
\begin{equation}
    \mathbb{E}_{C}\mathbb{E}_{\xi}[\mathcal{L}(\mathbf{w}^{\tau}, \boldsymbol{\hat{\alpha}}^{\tau},C^{*}) - \mathcal{L}(\mathbf{w}^{\tau}, \boldsymbol{\hat{\alpha}}^{\tau},C)] \leq \mathcal{L}(\mathbf{w}^{*}, \boldsymbol{\hat{\alpha}}^{*},C^{*}) -\mathcal{L}(\mathbf{w}^{0}, \boldsymbol{\hat{\alpha}}^{0},C^{*})+B,
\end{equation}
where 
\begin{equation}
    B= -\frac{EC^{*}\Delta_{\xi}^{2}}{4\eta}+ 2(E\eta^2L(\eta L+1)\sigma^2+2E^{3}\eta^{3}L^{2}G^{2}).
\end{equation}

\end{proof}

\section{Experiment}\label{sec:appendix Experiment}
\subsection{Detailed Experiment Settings}
\paragraph{Experiments compute resources.} All experiments are conducted using PyTorch on a node with two Intel Xeon Gold 6240 CPUs and one NVIDIA TITAN RTX GPU (24GB).

\paragraph{Concept drift settings.} There are 6 time steps. Each time step contains 40 communication rounds.  We set the client sampling rate to 0.5. There are 60 clients.
\begin{itemize}
    \item \textbf{At time step 1}, there are \textbf{two} decision boundaries(dataset conditional distribution) in the system. 30 clients will obtain data with original labels, while the labels of the other 30 clients will be modified from $y$ to $R-1-y$. $R$ is the number of class. The images in the clients are not rotated (\textbf{0 degree}). 
    \item \textbf{At time step 2}, there are \textbf{three} decision boundaries in the system. One-third of the clients will obtain data with original labels, while the labels of the other one-third clients will be modified from $y$ to $R-1-y$. For the remaining one-third clients, the labels will be modified from $y$ to $y+1 \bmod R$. The images in the clients are rotated \textbf{120 degrees}. 
    \item \textbf{At time step 3}, there are \textbf{four} decision boundaries in the system. One-fourth of the clients will obtain data with original labels, while the labels of another one-fourth clients will be modified, from $y$ to $R-1-y$. For the next one-fourth clients, the labels will be modified from $y$ to $y+1 \bmod R$. For the final one-fourth clients, the labels will be modified, from $y$ to $y+2 \bmod R$. The images in the clients are rotated \textbf{240 degrees}.
    \item \textbf{At time step 4}, there are \textbf{four} decision boundaries in the system. One-fourth of the clients will obtain data with original labels, while the labels of another one-fourth clients will be modified, from $y$ to $R-1-y$. For the next one-fourth clients, the labels will be modified from $y$ to $y+1 \bmod R$. For the final one-fourth clients, the labels will be modified, from $y$ to $y+2 \bmod R$. The images in the clients are not rotated (\textbf{0 degree}). 
    \item \textbf{At time step 5}, there are \textbf{four} decision boundaries in the system. One-fourth of the clients will obtain data with original labels, while the labels of another one-fourth clients will be modified, from $y$ to $R-1-y$. For the next one-fourth clients, the labels will be modified from $y$ to $y+1 \bmod R$. For the final one-fourth clients, the labels will be modified, from $y$ to $y+2 \bmod R$. The images in the clients are rotated \textbf{120 degrees}.  
    \item \textbf{At time step 6}, there are \textbf{four} decision boundaries in the system. One-fourth of the clients will obtain data with original labels, while the labels of another one-fourth clients will be modified, from $y$ to $R-1-y$. For the next one-fourth clients, the labels will be modified from $y$ to $y+1 \bmod R$. For the final one-fourth clients, the labels will be modified, from $y$ to $y+2 \bmod R$. The images in the clients are rotated \textbf{240 degrees}.
\end{itemize}

To provide a clearer and more intuitive understanding of the experimental setup, the settings are illustrated in Fig.~\ref{fig:Concept drift settings} and summarized in Table~\ref{table:Dataset setting of concept drift.}.

\begin{table}[t]
    \centering
    \caption{Concept drift settings.}
    \begin{tabular}{@{}c >{\centering\arraybackslash} p{2cm}>{\centering\arraybackslash} p{2cm}@{}}
        \toprule
        Time steps & Number of decision boundaries & Image rotated degrees \\
        \midrule
        
        1 & 2 & 0 \\
        2 & 3 & 120 \\
        3 & 4 & 240 \\
        4 & 4 & 0 \\
        5 & 4 & 120 \\
        6 & 4 & 240 \\
        \bottomrule
    \end{tabular}

    \label{table:Dataset setting of concept drift.}
\end{table}

\begin{figure}[t]
    \centering
    \includegraphics[width=1\textwidth]{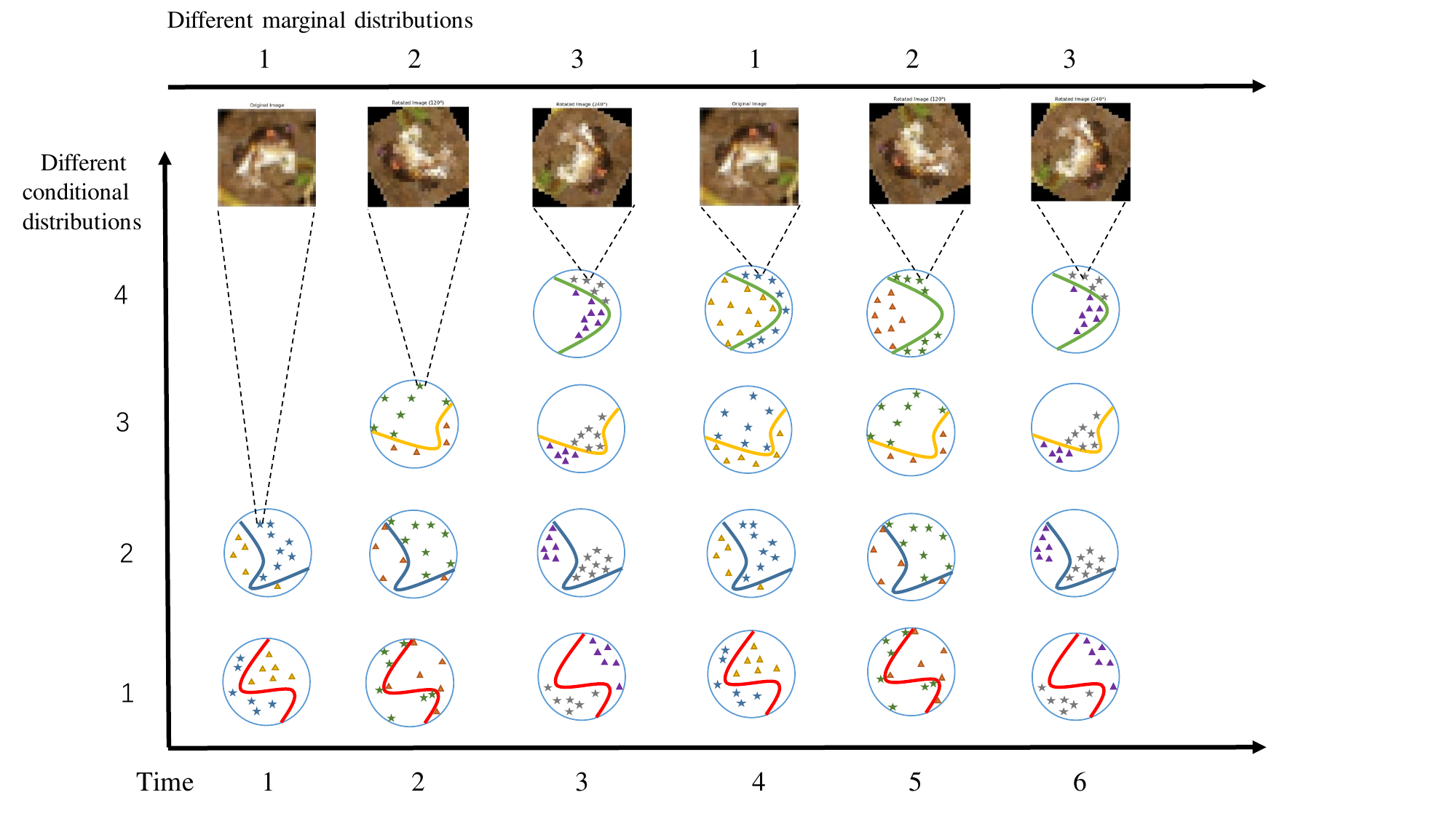}
    \caption{Visualization of the concept drift settings.}
    \label{fig:Concept drift settings}
\end{figure}

\paragraph{Hyperparameters.} We adopt SGD as the optimizer across all baselines, with a batch size of $128$ and momentum set to $0.9$. Unless otherwise specified, we use $E = 1$ local epoch per round for all datasets. Each algorithm is trained for 40 communication rounds, and we report the average test accuracy across all clients and all time steps.

FedAvg: The learning rate is tuned from $\{0.01, 0.03, 0.06, 0.1\}$ and set to $0.06$.

A-FedAvg: The learning rate is tuned from $\{0.01, 0.03, 0.06, 0.1\}$ and set to $0.06$.

Flash: The learning rate is tuned from $\{0.01, 0.03, 0.06, 0.1\}$ and set to $0.03$.

FedRC: The learning rate is tuned from $\{0.01, 0.03, 0.06, 0.1\}$ and set to $0.1$. The initial number of clusters is set to 2.

FedDrift: The learning rate is tuned from $\{0.01, 0.03, 0.06, 0.1\}$ and set to $0.06$. The initial number of clusters is set to 2.

FedDAA: The learning rate  is tuned from $\{0.01, 0.03, 0.06, 0.1\}$ and set to $0.06$. 

Oracle: The learning rate is tuned from $\{0.01, 0.03, 0.06, 0.1\}$ and set to $0.1$.

\subsection{Additional Experimental Results}
\begin{figure}[t]
    \centering
    \includegraphics[width=1\textwidth]{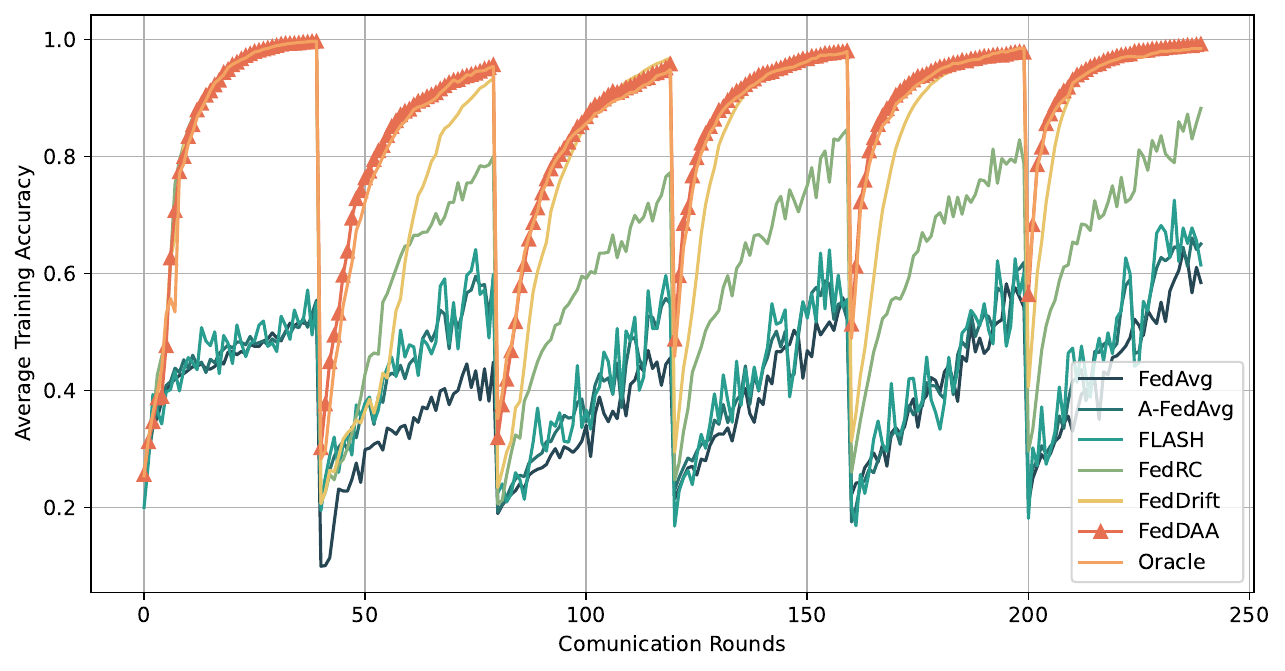}
    \caption{Average training accuracy of all methods on CIFAR-10 over 6 time steps. Each time step contains 40 communication rounds.}
    \label{fig:Average test accuracy over time.}
\end{figure}
\begin{figure}[!htbp]
    \centering
    \begin{minipage}{\textwidth}
        \centering
        \begin{minipage}{0.32\textwidth}
            \includegraphics[width=\linewidth]{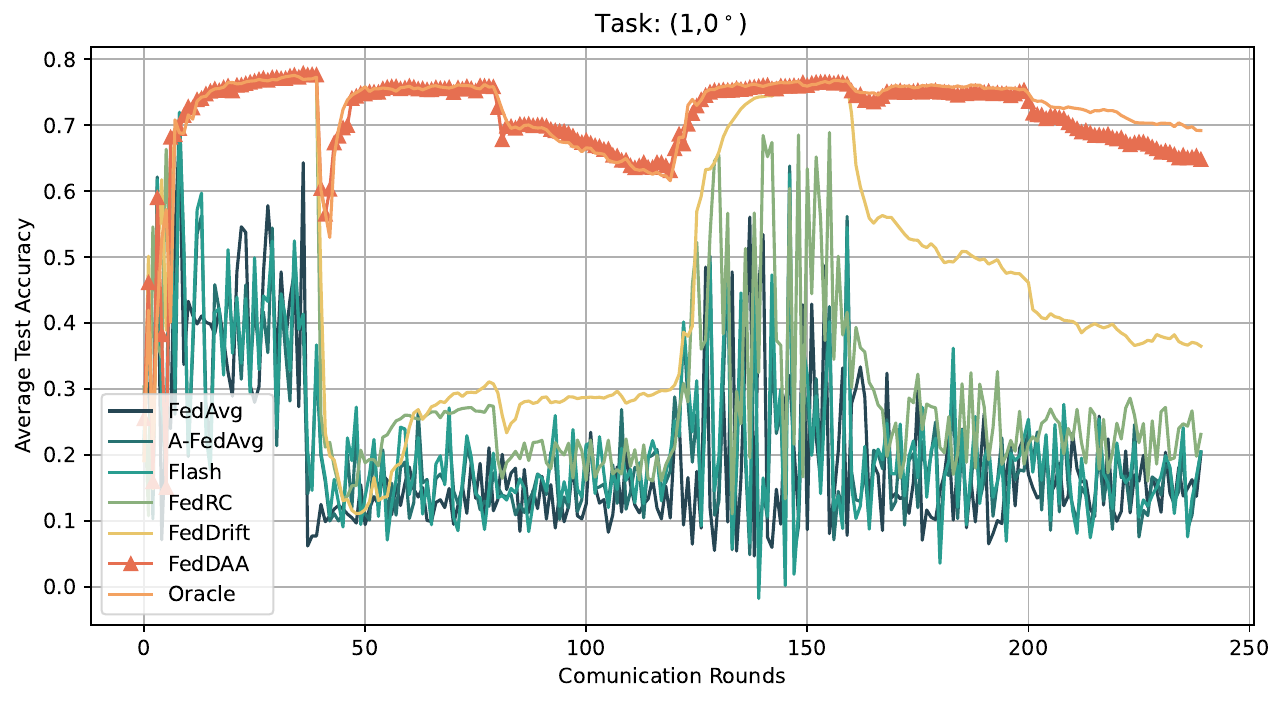}
            \subcaption{Task: (1, $0^\circ$)}
            \label{fig:12subfig-a}
        \end{minipage}
        \begin{minipage}{0.32\textwidth}
            \includegraphics[width=\linewidth]{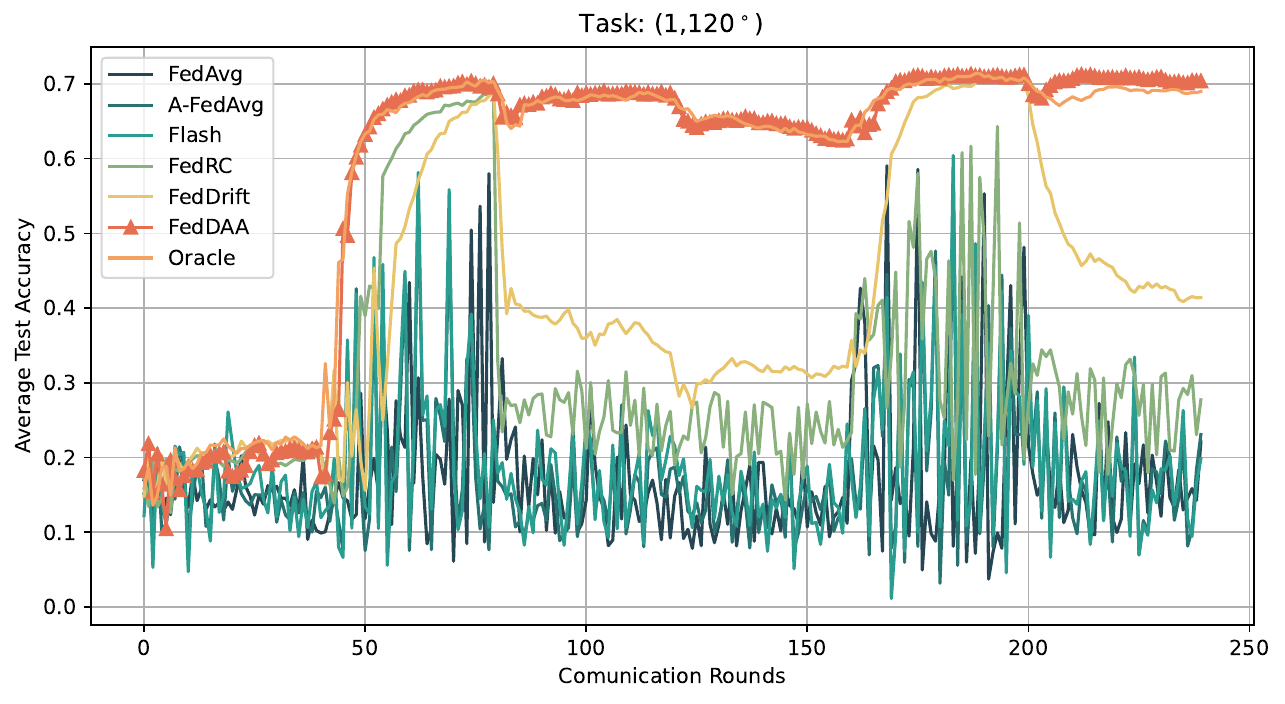}
            \subcaption{Task: (1, $120^\circ$)}
        \end{minipage}
        \begin{minipage}{0.32\textwidth}
            \includegraphics[width=\linewidth]{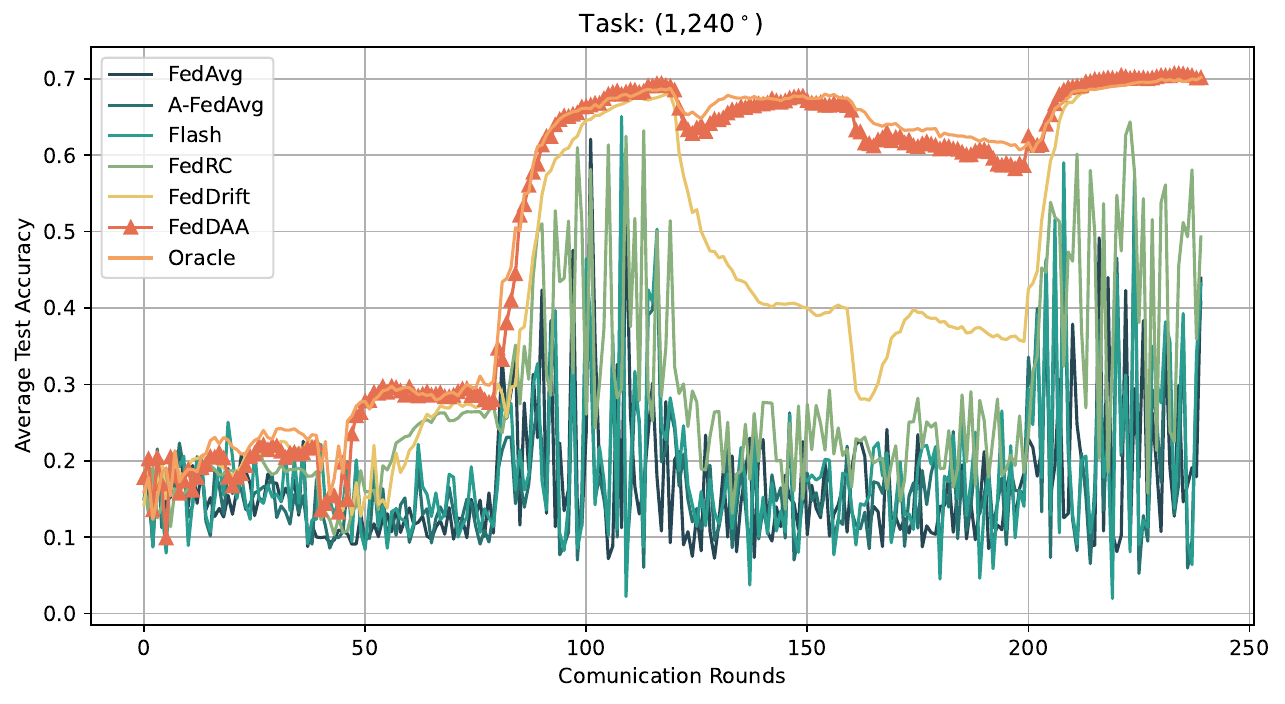}
            \subcaption{Task: (1, $240^\circ$)}
        \end{minipage}
    \end{minipage}
    

    \begin{minipage}{\textwidth}
        \centering
        \begin{minipage}{0.32\textwidth}
            \includegraphics[width=\linewidth]{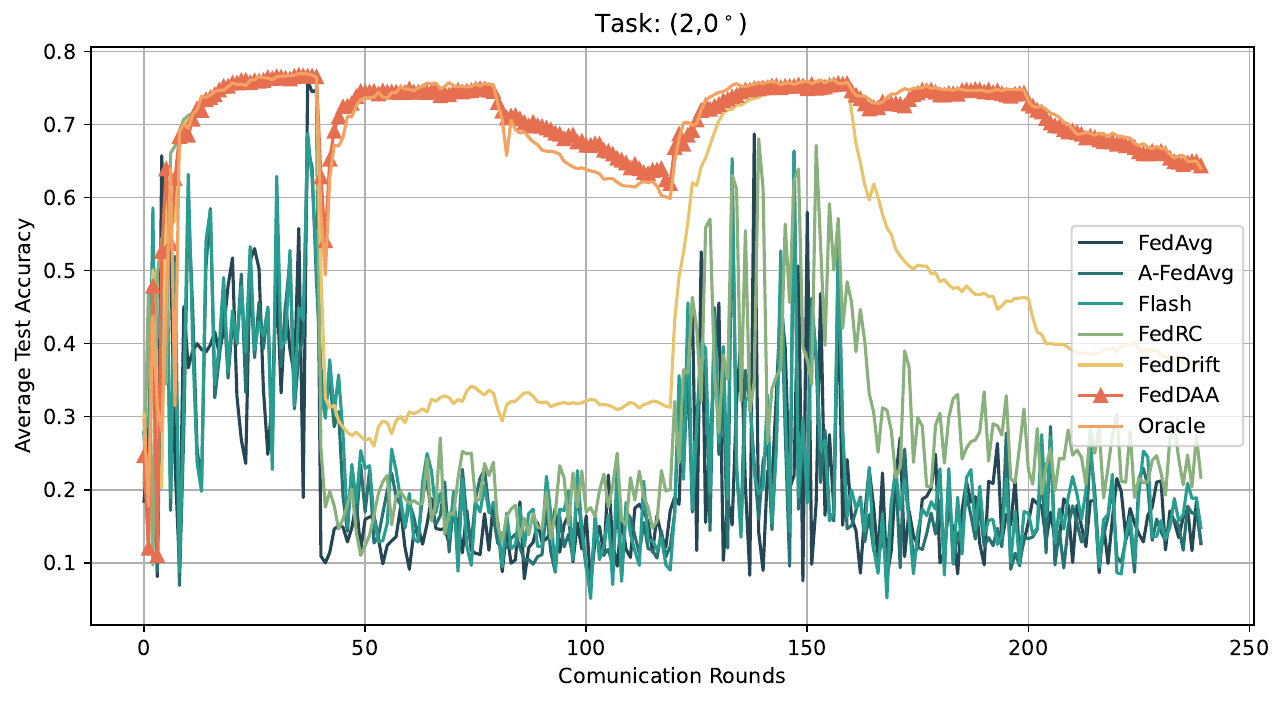}
            \subcaption{Task: (2, $0^\circ$)}
        \end{minipage}
        \begin{minipage}{0.32\textwidth}
            \includegraphics[width=\linewidth]{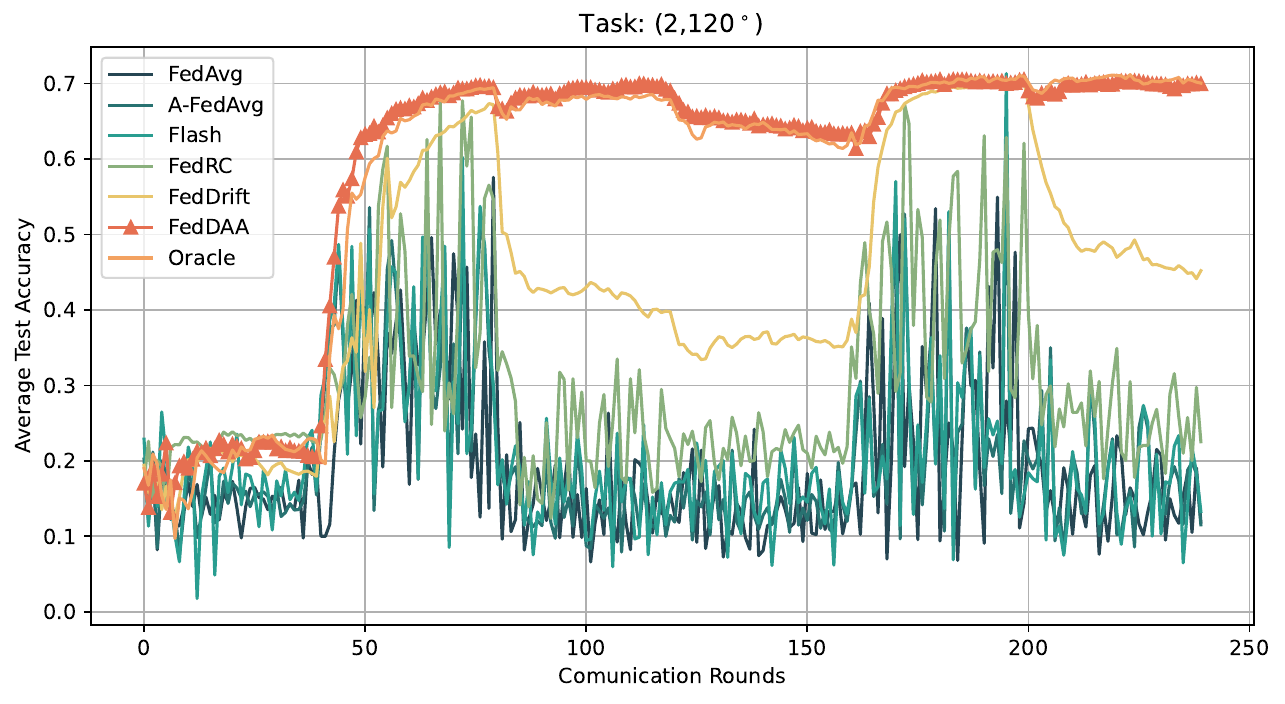}
            \subcaption{Task: (2, $120^\circ$)}
        \end{minipage}
        \begin{minipage}{0.32\textwidth}
            \includegraphics[width=\linewidth]{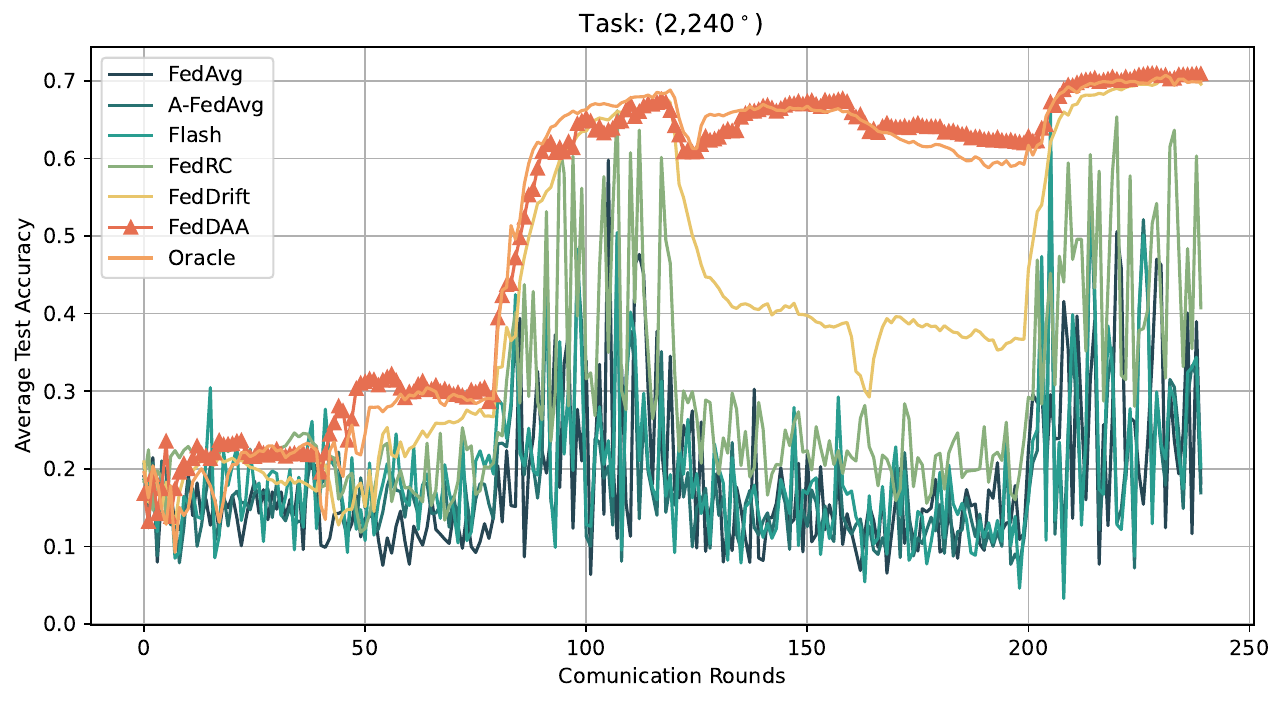}
            \subcaption{Task: (2, $240^\circ$)}
        \end{minipage}
    \end{minipage}
    

    \begin{minipage}{\textwidth}
        \centering
        \begin{minipage}{0.32\textwidth}
            \includegraphics[width=\linewidth]{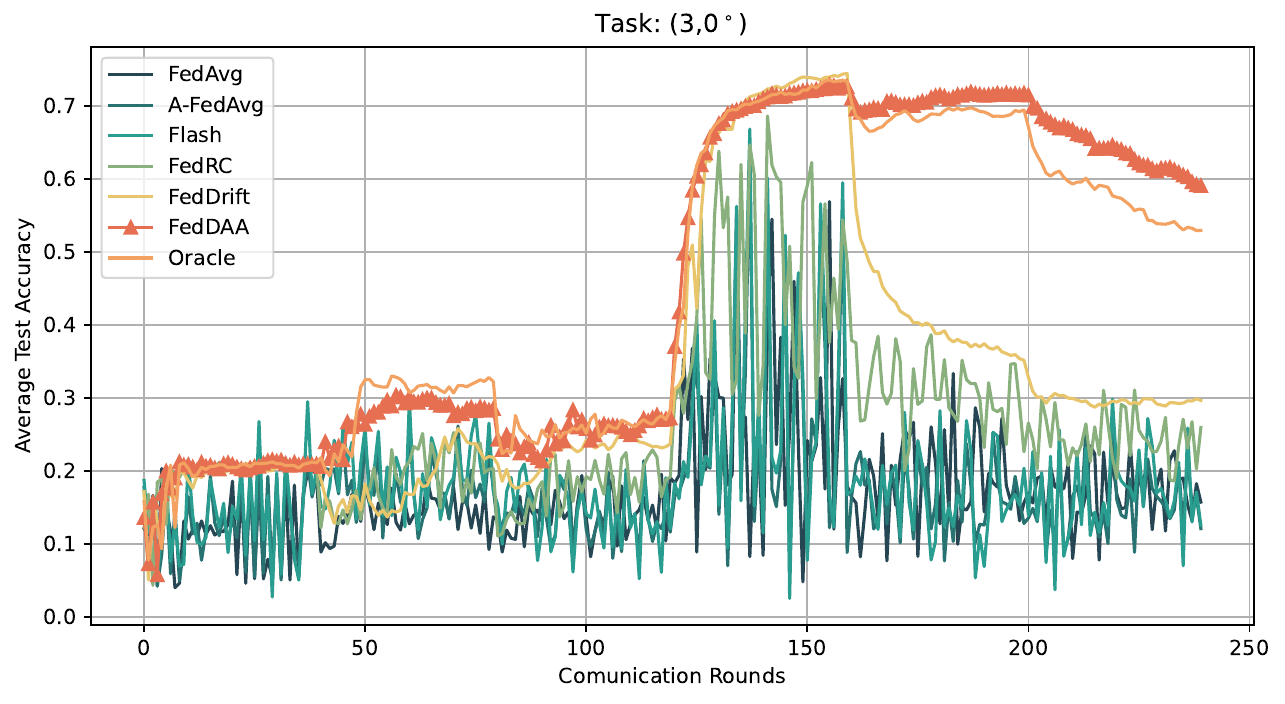}
            \subcaption{Task: (3, $0^\circ$)}
        \end{minipage}
        \begin{minipage}{0.32\textwidth}
            \includegraphics[width=\linewidth]{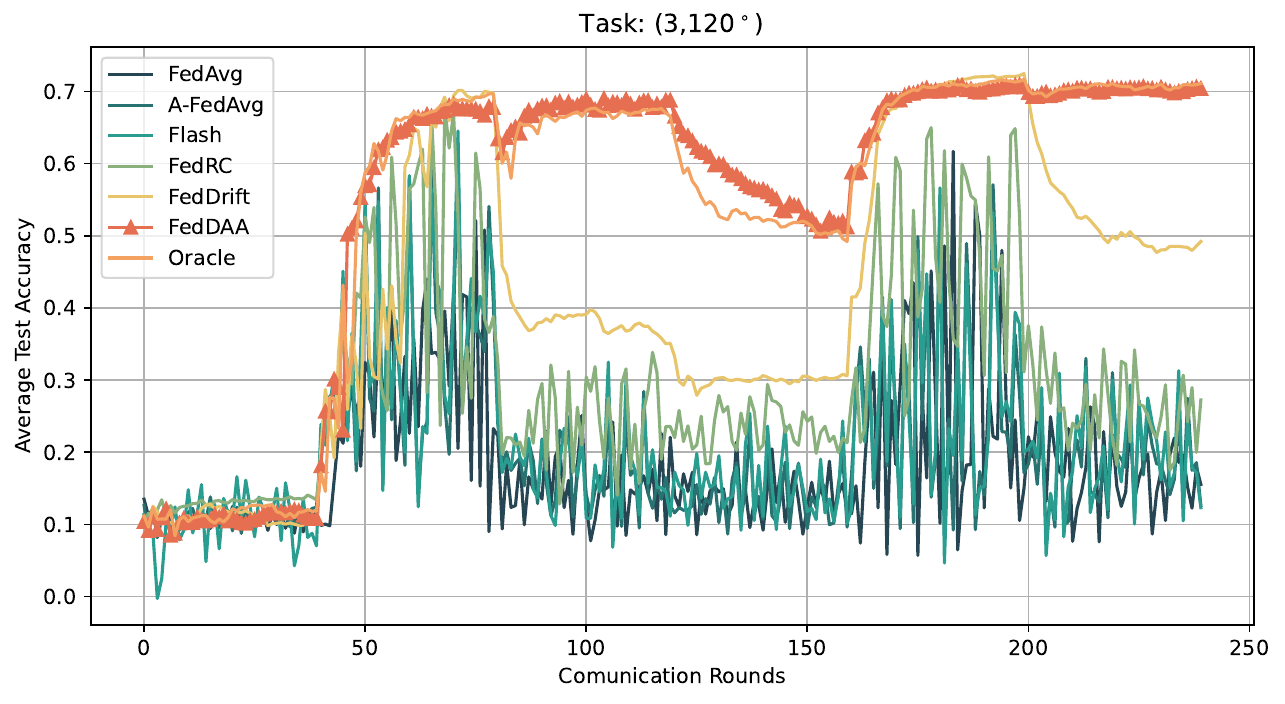}
            \subcaption{Task: (3, $120^\circ$)}
        \end{minipage}
        \begin{minipage}{0.32\textwidth}
            \includegraphics[width=\linewidth]{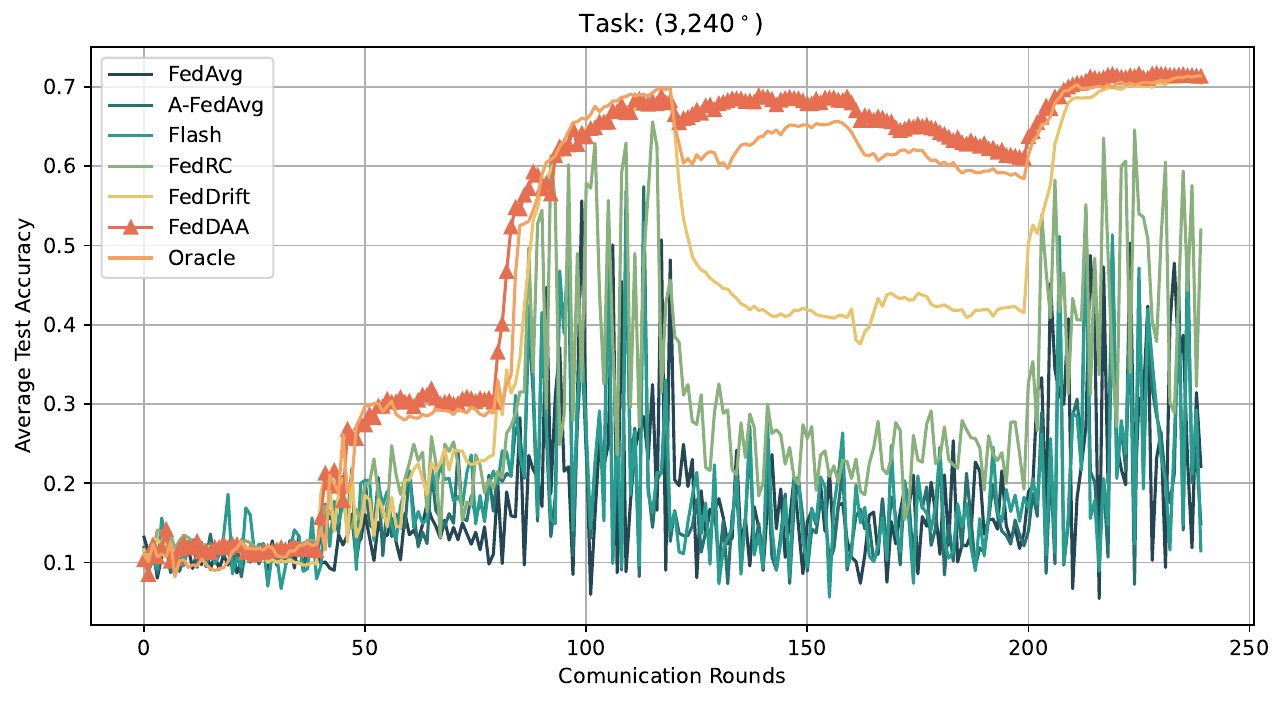}
            \subcaption{Task: (3, $240^\circ$)}
        \end{minipage}
    \end{minipage}
    

    \begin{minipage}{\textwidth}
        \centering
        \begin{minipage}{0.32\textwidth}
            \includegraphics[width=\linewidth]{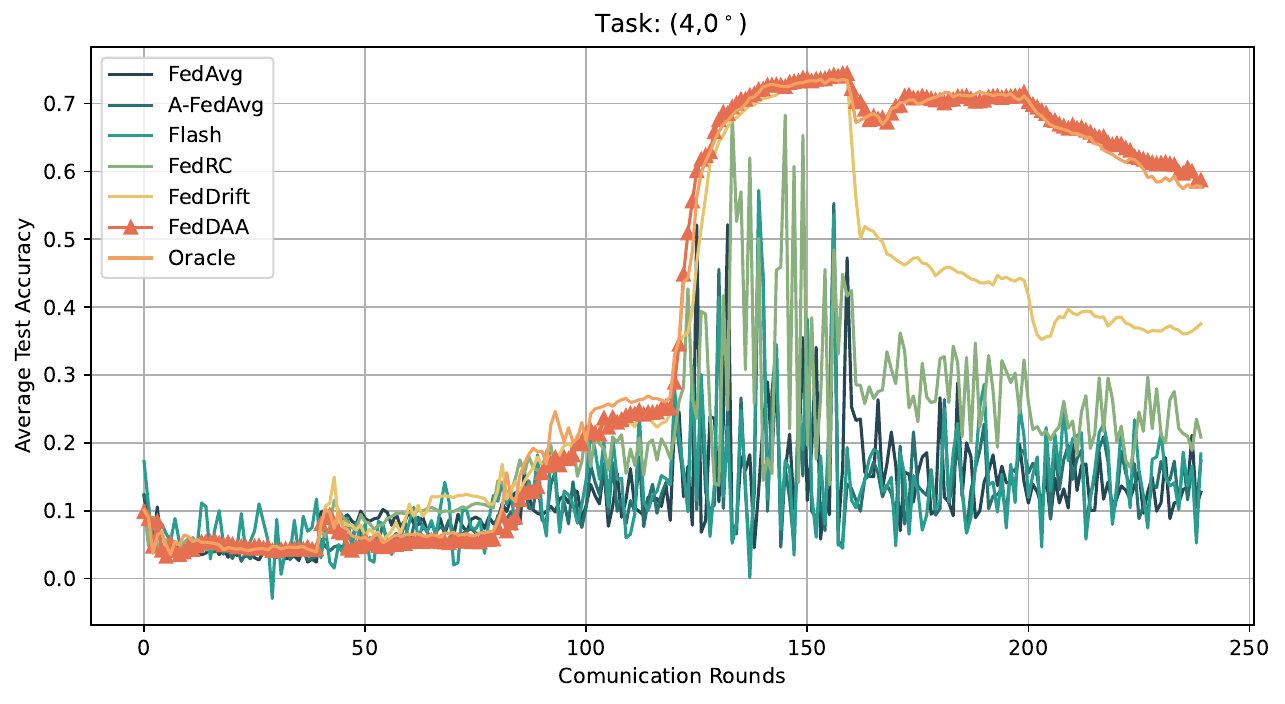}
            \subcaption{Task: (4, $0^\circ$)}
        \end{minipage}
        \begin{minipage}{0.32\textwidth}
            \includegraphics[width=\linewidth]{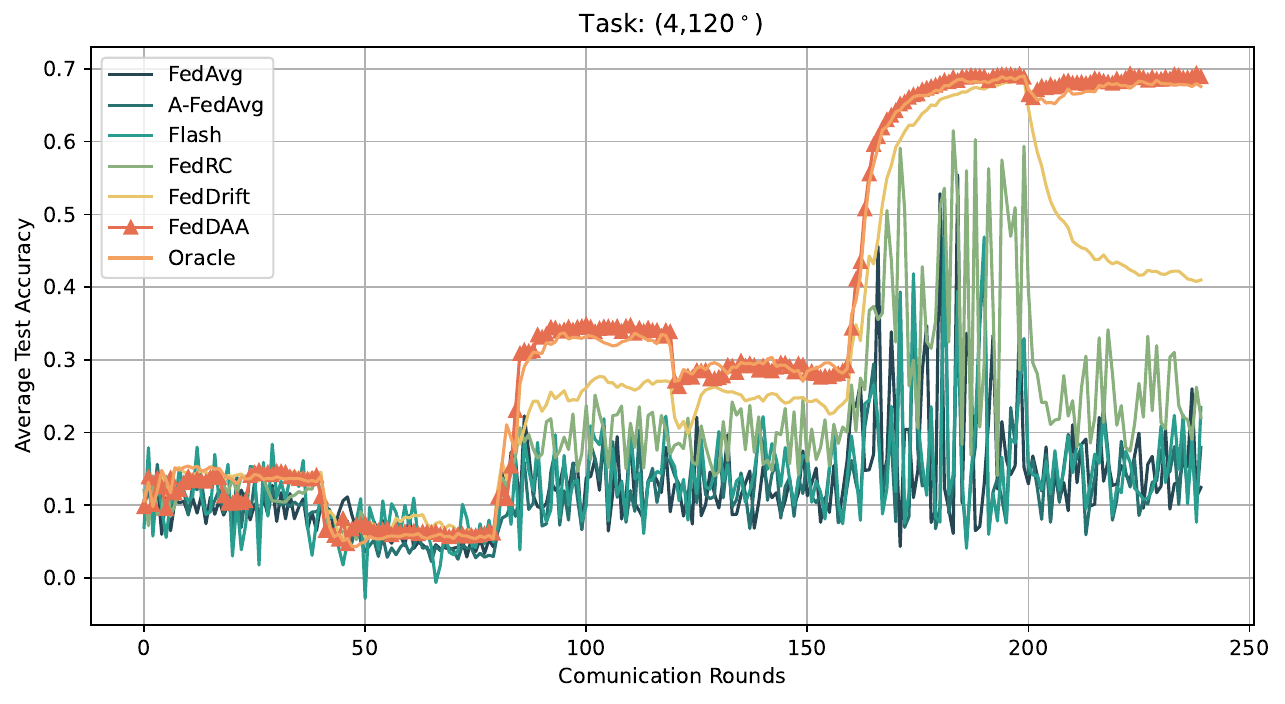}
            \subcaption{Task: (4, $120^\circ$)}
        \end{minipage}
        \begin{minipage}{0.32\textwidth}
            \includegraphics[width=\linewidth]{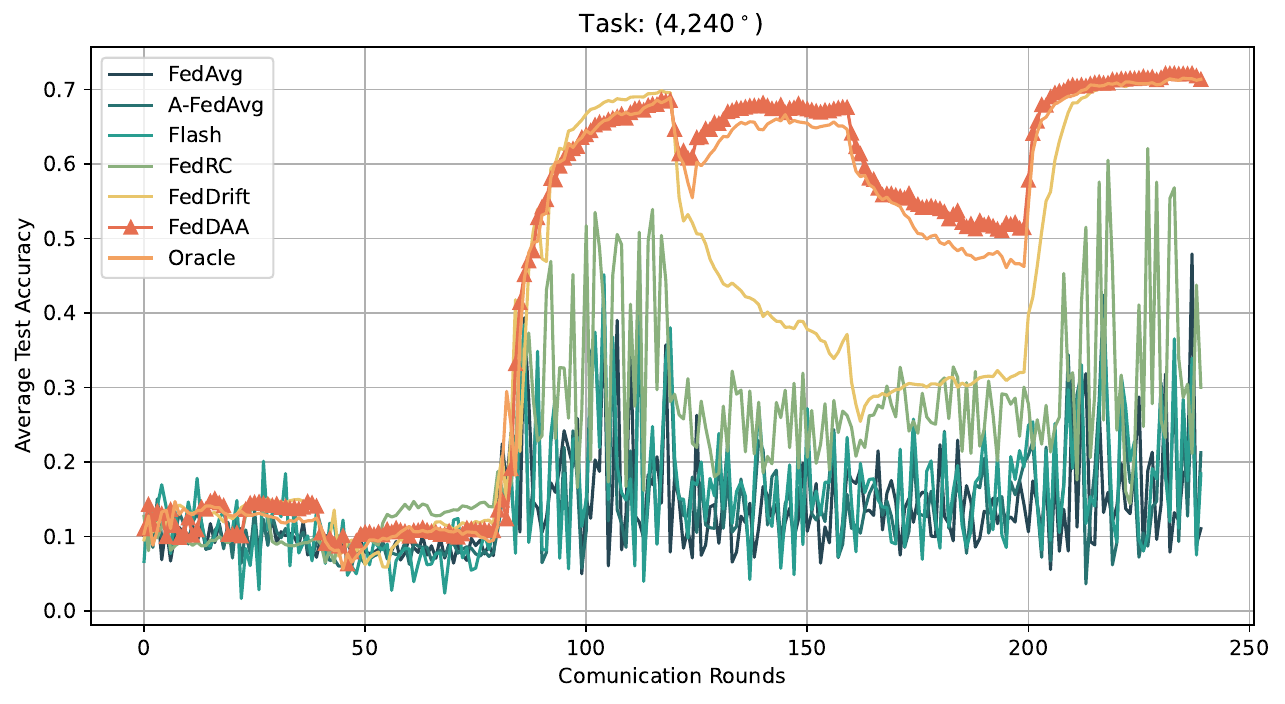}
            \subcaption{Task: (4, $240^\circ$)}
        \end{minipage}
    \end{minipage}

    \caption{Average test accuracy on CIFAR-10 for all tasks over 6 time steps. Each time step contains 40 communication rounds.}
    \label{fig:12-subfigs}
\end{figure}
\paragraph{Comparison with baselines using average accuracy over time.} Fig.~\ref{fig:Average test accuracy over time.} shows the average training accuracy on CIFAR-10 of all methods over 6 time steps. Our method FedDAA achieves the highest and most stable accuracy among all methods, closely approaching the Oracle baseline. What is more, FedDAA demonstrates superior robustness to concept drift, showing reduced performance degradation and alleviating catastrophic forgetting relative to other approaches. In particular, the results show that when multiple data conditional distributions coexist in the system, single-model approaches (e.g., FedAvg, A-FedAvg, and Flash) fail to converge and perform significantly worse than multiple-model methods (e.g., FedRC, FedDrift, and FedDAA). Moreover, FedRC struggles in our scenario, where the number of clusters changes over time, as it cannot dynamically adjust the number of clusters. Although FedDrift adapts relatively well to new data, we observed that it tends to use a larger number of clusters and suffers from catastrophic forgetting when concept drift occurs. These results demonstrate the effectiveness of FedDAA in handling multi-source concept drift and mitigating catastrophic forgetting.

\paragraph{Comparison with baselines on forgetting mitigation.} 

Fig.~\ref{fig:12-subfigs} presents the average test accuracy of all methods on 12 different tasks from CIFAR-10 across 6 time steps. Here, (1, $0^\circ$) denotes (concept 1, 0-degree rotation). "Concepts 1 to 4" refer to different conditional distributions. Across all 12 task configurations, our proposed FedDAA consistently and significantly outperforms all baseline methods.
Notably, FedDAA shows a strong ability to adapt to multi-source concept drift. It maintains stable performance and effectively mitigates catastrophic forgetting, whereas the baseline methods often experience substantial performance degradation when concept drift occurs. For example, in tasks such as $(1, 0^\circ)$, as shown in Fig.~\ref{fig:12subfig-a}, FedDAA sustains a test accuracy of approximately 0.70 to 0.75, while the performance of other methods suffers catastrophic forgetting due to the presence of concept drift.
These results highlight FedDAA’s advanced capability in handling both temporal and spatial data heterogeneity, confirming its effectiveness as a robust FL framework for dynamic and non-stationary environments.

%% file: my_neurips_2025.bbl
\begin{thebibliography}{37}
\providecommand{\natexlab}[1]{#1}
\providecommand{\url}[1]{\texttt{#1}}
\expandafter\ifx\csname urlstyle\endcsname\relax
  \providecommand{\doi}[1]{doi: #1}\else
  \providecommand{\doi}{doi: \begingroup \urlstyle{rm}\Url}\fi

\bibitem[Lu et~al.(2019)Lu, Liu, Dong, Gu, Gama, and Zhang]{DBLP:journals/tkde/LuLDGGZ19}
Jie Lu, Anjin Liu, Fan Dong, Feng Gu, Jo{\~{a}}o Gama, and Guangquan Zhang.
\newblock Learning under concept drift: {A} review.
\newblock \emph{{IEEE} Trans. Knowl. Data Eng.}, 31\penalty0 (12):\penalty0 2346--2363, 2019.

\bibitem[Kairouz et~al.(2021)Kairouz, McMahan, Avent, et~al.]{DBLP:journals/ftml/KairouzMABBBBCC21}
Peter Kairouz, H.~Brendan McMahan, Brendan Avent, et~al.
\newblock Advances and open problems in federated learning.
\newblock \emph{Found. Trends Mach. Learn.}, 14\penalty0 (1-2):\penalty0 1--210, 2021.

\bibitem[Gama et~al.(2014)Gama, Zliobaite, Bifet, Pechenizkiy, and Bouchachia]{DBLP:journals/csur/GamaZBPB14}
Jo{\~{a}}o Gama, Indre Zliobaite, Albert Bifet, Mykola Pechenizkiy, and Abdelhamid Bouchachia.
\newblock A survey on concept drift adaptation.
\newblock \emph{{ACM} Comput. Surv.}, 46\penalty0 (4):\penalty0 44:1--44:37, 2014.

\bibitem[Guo et~al.(2024)Guo, Tang, and Lin]{DBLP:conf/icml/GuoTL24}
Yongxin Guo, Xiaoying Tang, and Tao Lin.
\newblock Fedrc: Tackling diverse distribution shifts challenge in federated learning by robust clustering.
\newblock In \emph{Proceedings of Forty-first International Conference on Machine Learning, {ICML} 2024}, 2024.

\bibitem[Ghosh et~al.(2020)Ghosh, Chung, Yin, and Ramchandran]{DBLP:conf/nips/GhoshCYR20}
Avishek Ghosh, Jichan Chung, Dong Yin, and Kannan Ramchandran.
\newblock An efficient framework for clustered federated learning.
\newblock In \emph{Proceedings of Advances in Neural Information Processing Systems 33: Annual Conference on Neural Information Processing Systems 2020, NeurIPS 2020}, 2020.

\bibitem[Wang et~al.(2023)Wang, Xu, Liu, Xu, Huang, and Zhao]{DBLP:journals/tmc/WangXLXHZ23}
Zhiyuan Wang, Hongli Xu, Jianchun Liu, Yang Xu, He~Huang, and Yangming Zhao.
\newblock Accelerating federated learning with cluster construction and hierarchical aggregation.
\newblock \emph{{IEEE} Trans. Mob. Comput.}, 22\penalty0 (7):\penalty0 3805--3822, 2023.

\bibitem[Briggs et~al.(2020)Briggs, Fan, and Andras]{DBLP:conf/ijcnn/BriggsFA20}
Christopher Briggs, Zhong Fan, and Peter Andras.
\newblock Federated learning with hierarchical clustering of local updates to improve training on non-iid data.
\newblock In \emph{Proceedings of 2020 International Joint Conference on Neural Networks, {IJCNN} 2020}, pages 1--9, 2020.

\bibitem[Sattler et~al.(2021)Sattler, M{\"{u}}ller, and Samek]{DBLP:journals/tnn/SattlerMS21}
Felix Sattler, Klaus{-}Robert M{\"{u}}ller, and Wojciech Samek.
\newblock Clustered federated learning: Model-agnostic distributed multitask optimization under privacy constraints.
\newblock \emph{{IEEE} Trans. Neural Networks Learn. Syst.}, 32\penalty0 (8):\penalty0 3710--3722, 2021.

\bibitem[Ruan and Joe{-}Wong(2022)]{DBLP:conf/aaai/RuanJ22}
Yichen Ruan and Carlee Joe{-}Wong.
\newblock Fedsoft: Soft clustered federated learning with proximal local updating.
\newblock In \emph{Proceedings of Thirty-Sixth {AAAI} Conference on Artificial Intelligence, {AAAI} 2022, Thirty-Fourth Conference on Innovative Applications of Artificial Intelligence, {IAAI} 2022, The Twelveth Symposium on Educational Advances in Artificial Intelligence, {EAAI} 2022}, pages 8124--8131, 2022.

\bibitem[Marfoq et~al.(2021)Marfoq, Neglia, Bellet, Kameni, and Vidal]{DBLP:conf/nips/MarfoqNBKV21}
Othmane Marfoq, Giovanni Neglia, Aur{\'{e}}lien Bellet, Laetitia Kameni, and Richard Vidal.
\newblock Federated multi-task learning under a mixture of distributions.
\newblock In \emph{Proceedings of Advances in Neural Information Processing Systems 34: Annual Conference on Neural Information Processing Systems 2021, NeurIPS 2021}, pages 15434--15447, 2021.

\bibitem[Long et~al.(2023)Long, Xie, Shen, Zhou, Wang, and Jiang]{DBLP:journals/www/LongXSZWJ23}
Guodong Long, Ming Xie, Tao Shen, Tianyi Zhou, Xianzhi Wang, and Jing Jiang.
\newblock Multi-center federated learning: clients clustering for better personalization.
\newblock \emph{World Wide Web {(WWW)}}, 26\penalty0 (1):\penalty0 481--500, 2023.

\bibitem[Stallmann and Wilbik(2022)]{DBLP:journals/corr/abs-2201-07316}
Morris Stallmann and Anna Wilbik.
\newblock Towards federated clustering: {A} federated fuzzy c-means algorithm {(FFCM)}.
\newblock \emph{CoRR}, abs/2201.07316, 2022.

\bibitem[Yan et~al.(2024)Yan, Tong, and Wang]{DBLP:journals/tnn/YanTW24}
Yihan Yan, Xiaojun Tong, and Shen Wang.
\newblock Clustered federated learning in heterogeneous environment.
\newblock \emph{{IEEE} Trans. Neural Networks Learn. Syst.}, 35\penalty0 (9):\penalty0 12796--12809, 2024.

\bibitem[Jothimurugesan et~al.(2023)Jothimurugesan, Hsieh, Wang, Joshi, and Gibbons]{DBLP:conf/aistats/JothimurugesanH23}
Ellango Jothimurugesan, Kevin Hsieh, Jianyu Wang, Gauri Joshi, and Phillip~B. Gibbons.
\newblock Federated learning under distributed concept drift.
\newblock In \emph{Proceedings of International Conference on Artificial Intelligence and Statistics}, volume 206 of \emph{Proceedings of Machine Learning Research}, pages 5834--5853, 2023.

\bibitem[Yang et~al.(2024)Yang, Chen, Zhang, and Wang]{yang2024multi}
Guanhui Yang, Xiaoting Chen, Tengsen Zhang, and Shuo Wang.
\newblock A multi-model approach for handling concept drifting data in federated learning.
\newblock In \emph{Workshop of Distributed Machine Learning and Unlearning for Sensor-Cloud Systems (DLS2) in the 20th International Conference on Mobility, Sensing and Networking}, 2024.

\bibitem[Chen et~al.(2024)Chen, Xue, Wang, Liu, and Huang]{DBLP:conf/nips/ChenX0LH24}
Junbao Chen, Jingfeng Xue, Yong Wang, Zhenyan Liu, and Lu~Huang.
\newblock Classifier clustering and feature alignment for federated learning under distributed concept drift.
\newblock In \emph{Proceedings of Advances in Neural Information Processing Systems 38: Annual Conference on Neural Information Processing Systems 2024, NeurIPS 2024}, 2024.

\bibitem[Zhou et~al.(2024)Zhou, Shekhar, Chhokra, Dubey, and Gokhale]{zhou2024drift}
Shuang Zhou, Shashank Shekhar, Ajay Chhokra, Abhishek Dubey, and Aniruddha Gokhale.
\newblock Drift detection and adaptation for federated learning in iot with adaptive device management.
\newblock In \emph{2024 IEEE International Conference on Big Data (BigData)}, pages 8088--8097. IEEE, 2024.

\bibitem[Casado et~al.(2022)Casado, Lema, Criado, Iglesias, Regueiro, and Barro]{DBLP:journals/mta/CasadoLCIRB22}
Fernando~E. Casado, Dylan Lema, Marcos~F. Criado, Roberto Iglesias, Carlos~V{\'{a}}zquez Regueiro, and Sen{\'{e}}n Barro.
\newblock Concept drift detection and adaptation for federated and continual learning.
\newblock \emph{Multim. Tools Appl.}, 81\penalty0 (3):\penalty0 3397--3419, 2022.

\bibitem[Chen et~al.(2021)Chen, Chai, Cheng, and Rangwala]{DBLP:conf/bigdataconf/ChenCCR21}
Yujing Chen, Zheng Chai, Yue Cheng, and Huzefa Rangwala.
\newblock Asynchronous federated learning for sensor data with concept drift.
\newblock In \emph{Proceedings of 2021 {IEEE} International Conference on Big Data (Big Data)}, pages 4822--4831, 2021.

\bibitem[Zhang et~al.(2024)Zhang, Zou, Xie, Zhang, Li, Cai, Cheng, and Yu]{DBLP:conf/mobihoc/ZhangZXZ0CCY24}
Ruirui Zhang, Yifei Zou, Zhenzhen Xie, Xiao Zhang, Peng Li, Zhipeng Cai, Xiuzhen Cheng, and Dongxiao Yu.
\newblock Federating from history in streaming federated learning.
\newblock In \emph{Proceedings of the Twenty-fifth International Symposium on Theory, Algorithmic Foundations, and Protocol Design for Mobile Networks and Mobile Computing, {MOBIHOC} 2024}, pages 151--160, 2024.

\bibitem[Canonaco et~al.(2021)Canonaco, Bergamasco, Mongelluzzo, and Roveri]{DBLP:conf/ijcnn/CanonacoBMR21}
Giuseppe Canonaco, Alex Bergamasco, Alessio Mongelluzzo, and Manuel Roveri.
\newblock Adaptive federated learning in presence of concept drift.
\newblock In \emph{Proceedings of International Joint Conference on Neural Networks, {IJCNN} 2021}, pages 1--7, 2021.

\bibitem[Panchal et~al.(2023)Panchal, Choudhary, Mitra, Mukherjee, Sarkhel, Mitra, and Guan]{DBLP:conf/icml/PanchalCMMSMG23}
Kunjal Panchal, Sunav Choudhary, Subrata Mitra, Koyel Mukherjee, Somdeb Sarkhel, Saayan Mitra, and Hui Guan.
\newblock Flash: Concept drift adaptation in federated learning.
\newblock In \emph{International Conference on Machine Learning, {ICML} 2023}, volume 202 of \emph{Proceedings of Machine Learning Research}, pages 26931--26962, 2023.

\bibitem[Saile et~al.(2024)Saile, Thomas, Kaaser, and Schulte]{saile2024client}
Finn Saile, Julius Thomas, Dominik Kaaser, and Stefan Schulte.
\newblock Client-side adaptation to concept drift in federated learning.
\newblock In \emph{Proceedings of 2nd International Conference on Federated Learning Technologies and Applications (FLTA)}, pages 71--78. IEEE, 2024.

\bibitem[Chen et~al.(2019)Chen, Liu, Sun, and Hong]{DBLP:conf/iclr/ChenLSH19}
Xiangyi Chen, Sijia Liu, Ruoyu Sun, and Mingyi Hong.
\newblock On the convergence of {A} class of adam-type algorithms for non-convex optimization.
\newblock In \emph{Proceedings of The 7th International Conference on Learning Representations, {ICLR} 2019}, 2019.

\bibitem[Mertikopoulos et~al.(2020)Mertikopoulos, Hallak, Kavis, and Cevher]{DBLP:conf/nips/MertikopoulosHK20}
Panayotis Mertikopoulos, Nadav Hallak, Ali Kavis, and Volkan Cevher.
\newblock On the almost sure convergence of stochastic gradient descent in non-convex problems.
\newblock In \emph{Proceedings of Advances in Neural Information Processing Systems 33: Annual Conference on Neural Information Processing Systems 2020, NeurIPS 2020}, 2020.

\bibitem[Li et~al.(2020)Li, Huang, Yang, Wang, and Zhang]{DBLP:conf/iclr/LiHYWZ20}
Xiang Li, Kaixuan Huang, Wenhao Yang, Shusen Wang, and Zhihua Zhang.
\newblock On the convergence of fedavg on non-iid data.
\newblock In \emph{Proceedings of The 8th International Conference on Learning Representations, {ICLR} 2020}, 2020.

\bibitem[Wang et~al.(2020)Wang, Liu, Liang, Joshi, and Poor]{DBLP:conf/nips/WangLLJP20}
Jianyu Wang, Qinghua Liu, Hao Liang, Gauri Joshi, and H.~Vincent Poor.
\newblock Tackling the objective inconsistency problem in heterogeneous federated optimization.
\newblock In \emph{Proceedings of Advances in Neural Information Processing Systems 33: Annual Conference on Neural Information Processing Systems 2020, NeurIPS 2020}, 2020.

\bibitem[Xiao et~al.(2017)Xiao, Rasul, and Vollgraf]{DBLP:journals/corr/abs-1708-07747}
Han Xiao, Kashif Rasul, and Roland Vollgraf.
\newblock Fashion-mnist: a novel image dataset for benchmarking machine learning algorithms.
\newblock \emph{CoRR}, abs/1708.07747, 2017.

\bibitem[Krizhevsky et~al.(2009)Krizhevsky, Hinton, et~al.]{krizhevsky2009learning}
Alex Krizhevsky, Geoffrey Hinton, et~al.
\newblock Learning multiple layers of features from tiny images.
\newblock 2009.

\bibitem[He et~al.(2016)He, Zhang, Ren, and Sun]{DBLP:conf/cvpr/HeZRS16}
Kaiming He, Xiangyu Zhang, Shaoqing Ren, and Jian Sun.
\newblock Deep residual learning for image recognition.
\newblock In \emph{Proceedings of 2016 {IEEE} Conference on Computer Vision and Pattern Recognition, {CVPR} 2016}, pages 770--778, 2016.

\bibitem[Howard et~al.(2017)Howard, Zhu, Chen, Kalenichenko, Wang, Weyand, Andreetto, and Adam]{DBLP:journals/corr/HowardZCKWWAA17}
Andrew~G. Howard, Menglong Zhu, Bo~Chen, Dmitry Kalenichenko, Weijun Wang, Tobias Weyand, Marco Andreetto, and Hartwig Adam.
\newblock Mobilenets: Efficient convolutional neural networks for mobile vision applications.
\newblock \emph{CoRR}, abs/1704.04861, 2017.

\bibitem[McMahan et~al.(2017)McMahan, Moore, Ramage, Hampson, and y~Arcas]{DBLP:conf/aistats/McMahanMRHA17}
Brendan McMahan, Eider Moore, Daniel Ramage, Seth Hampson, and Blaise~Ag{\"{u}}era y~Arcas.
\newblock Communication-efficient learning of deep networks from decentralized data.
\newblock In \emph{Proceedings of the 20th International Conference on Artificial Intelligence and Statistics, {AISTATS} 2017}, volume~54 of \emph{Proceedings of Machine Learning Research}, pages 1273--1282, 2017.

\bibitem[Nguyen et~al.(2022{\natexlab{a}})Nguyen, Tran, Gal, Torr, and Baydin]{DBLP:conf/iclr/NguyenTGTB22}
A.~Tuan Nguyen, Toan Tran, Yarin Gal, Philip H.~S. Torr, and Atilim~Gunes Baydin.
\newblock {KL} guided domain adaptation.
\newblock In \emph{Proceedings of The Tenth International Conference on Learning Representations, {ICLR} 2022}, 2022{\natexlab{a}}.

\bibitem[Nguyen et~al.(2022{\natexlab{b}})Nguyen, Torr, and Lim]{DBLP:conf/nips/NguyenTL22}
A.~Tuan Nguyen, Philip H.~S. Torr, and Ser~Nam Lim.
\newblock Fedsr: {A} simple and effective domain generalization method for federated learning.
\newblock In \emph{Proceedings of Advances in Neural Information Processing Systems 35: Annual Conference on Neural Information Processing Systems 2022, NeurIPS 2022}, 2022{\natexlab{b}}.

\bibitem[Wang and Mao(2023)]{DBLP:conf/iclr/WangM23}
Ziqiao Wang and Yongyi Mao.
\newblock Information-theoretic analysis of unsupervised domain adaptation.
\newblock In \emph{Proceedings of The Eleventh International Conference on Learning Representations, {ICLR} 2023}, 2023.

\bibitem[Pham et~al.(2024)Pham, Zhang, and Zhang]{DBLP:conf/uai/PhamZ024}
Thai{-}Hoang Pham, Xueru Zhang, and Ping Zhang.
\newblock Non-stationary domain generalization: Theory and algorithm.
\newblock In \emph{Uncertainty in Artificial Intelligence}, volume 244, pages 2902--2927, 2024.

\bibitem[Hsu et~al.(2019)Hsu, Qi, and Brown]{DBLP:journals/corr/abs-1909-06335}
Tzu{-}Ming~Harry Hsu, Hang Qi, and Matthew Brown.
\newblock Measuring the effects of non-identical data distribution for federated visual classification.
\newblock \emph{CoRR}, abs/1909.06335, 2019.

\end{thebibliography}
